\documentclass[]{llncs}
\usepackage[T1]{fontenc}
\usepackage{graphicx}

\usepackage[hidelinks]{hyperref}
\usepackage{color}

\usepackage[american]{babel}

\usepackage{amsmath}
\usepackage{amssymb}
\usepackage{algorithm}
\usepackage{algpseudocode}
\usepackage{stmaryrd}
\usepackage[compatibility=false]{caption}
\usepackage{subcaption}
\usepackage[export]{adjustbox}
\usepackage{enumitem}
\usepackage{fullpage}
\usepackage{microtype}

\newcommand{\counterexample}{e}
\newcommand{\counterexamples}{E}

\begin{document}
\title{Reinforcement Learning with \texorpdfstring{\\}{ } Symbolic Reward Machines}
\author{Thomas Krug\inst{1,2} \and
Daniel Neider\inst{1,2}}

\institute{TU Dortmund University \and Research Center Trustworthy Data Science and Security}
\maketitle
\begin{abstract}
Reward Machines (RMs) are an established mechanism in Reinforcement Learning (RL) to represent and learn sparse, temporally extended tasks with non-Markovian rewards. RMs rely on high-level information in the form of labels that are emitted by the environment alongside the observation. However, this concept requires manual user input for each environment and task. The user has to create a suitable labeling function that computes the labels. These limitations lead to poor applicability in widely adopted RL frameworks.
We propose Symbolic Reward Machines (SRMs) together with the learning algorithms QSRM and LSRM to overcome the limitations of RMs. SRMs consume only the standard output of the environment and process the observation directly through guards that are represented by symbolic formulas. In our evaluation, our SRM methods outperform the baseline RL approaches and generate the same results as the existing RM methods. At the same time, our methods adhere to the widely used environment definition and provide interpretable representations of the task to the user.

\end{abstract}
\section{Introduction}
Reinforcement Learning (RL) is an established mechanism to train agents to achieve a task in an environment that only outputs a state, which is also called observation, and a reward. A function, which is called reward function, is used to generate the reward. The reward indicates how good the last move of the agent was. The reward function is Markovian in the classical definition of RL \cite{Sutton.2020}. This means that the reward depends only on the last state and the action taken by the agent.
However, this is for many practical tasks a critical restriction because real-world problems often consist of a sequence of steps to accomplish the task's goal. As an example, if an agent should fill a certain machine with wood, the agent has first to collect wood before the agent can put it in the machine. To define these tasks, non-Markovian reward functions are required. Non-Markovian reward functions map state histories to rewards. In other words, they include temporal dependencies.

Reward Machines (RMs) \cite{Icarte.2018,10.1613/jair.1.12440} are one approach to represent non-Markovian reward functions. Icarte et al. \cite{Icarte.2018} developed QRM that uses RMs to train agents. The idea of QRM is to use one Q-table for each RM state and to store the current RM state. Thus, QRM requires the assumption that the RMs are given a priori. Furthermore, QRM accelerates the learning process by performing a multi-update step: QRM not only updates the Q-value of the current RM state but also updates the Q-values for all RM states corresponding to the observed environment state. This is possible because QRM has access to the RM and therefore knows which reward the agent would receive if the RM were in any of its other RM states. JIRP \cite{Xu_Gavran_Ahmad_Majumdar_Neider_Topcu_Wu_2020} lifts this assumption and infers the RM on the fly along with the policy.

However, RMs generally require an extended variant of the broadly used interaction scheme between the agent and the environment. The environment has to output high-level events which are used as inputs for the RM. These high-level events are generated by a labeling function. Aside from requiring external expert knowledge, implementing the labeling function is a technical obstacle that prevents using most RL environments out of the box. Furthermore, on the one hand, the labeling function has to be highly generic for the environment to support a variety of tasks. On the other hand, it has to be specific enough for a concrete task. Consequently, the labeling function either needs to be continuously extended, or it has to be designed to cover all possible tasks. However, both strategies result in poor usability. In addition, RMs cannot even represent Markovian tasks without an appropriate labeling function. As an example, if you get a reward at a specific state, the labeling function has to generate a specific high-level event for this state. Otherwise, it is not possible to design an RM that outputs the correct reward for this state.

We propose Symbolic Reward Machines (SRMs) to lift the requirement of the labeling function and the deviation from the broadly used interaction scheme between the agent and the environment. SRMs provide the ability to represent non-Markovian reward functions without an explicit labeling function.
SRMs have internal states to represent specific information, e.g., the task completion. The agent can then make decisions based on these internal states. In contrast to RMs, SRMs do not require abstract high-level events from a labeling function. Instead, SRMs use symbolic formulas as guards at the transitions. These guards can directly process the environment state.
On the basis of SRMs, we provide RL learning algorithms that support the usage of SRMs. We call them QSRM and LSRM. QSRM requires the SRM as input. LSRM is an extension of QSRM and learns the SRM in the training process.
Thus, LSRM enables the possibility to learn an effective policy end-to-end again. In this context, end-to-end refers to the ability to train from a standard environment to an effective policy. The idea of LSRM is to learn the SRM from experiences. By learning the SRM, LSRM provides valuable and interpretable information about the reward structure. For example, from the learned SRM, the user gains insight into the structure of the task and obtains information about what the agent needs to do to accomplish the task's goal. The learned SRM provides a step for step explanation for the task.
In the evaluation part, we demonstrate the effectiveness of SRMs and our new algorithms in multiple experiments including finite (discrete) and infinite (continuous) state spaces. In addition, we compare the results to existing algorithms, e.g., QRM, to show the advantage of our novel methods. Furthermore, the results demonstrate that we can train effective policies end-to-end, and thus, respect the standard definition for RL in the case of non-Markovian reward functions.

Beside RMs, there is further related work in the area of non-Markovian reward functions in RL.
Zhou and Li \cite{pmlr-v162-zhou22b} introduced a concept that they also name `symbolic reward machines'. However, their formal definition differs to our definition of a Symbolic Reward Machine. First, they define the guards as a set of predicates on trajectories $\Psi$, where $\Psi \subseteq (S \times A)^* \rightarrow \mathbb{B}$. This means that a guard in a state of the SRM can access the whole trajectory seen so far. In addition, the reward, which the SRM outputs in a state, can also depend on the whole trajectory seen so far because they define the reward function as a set of $\lambda$-terms of type $(S \times A)^* \rightarrow \mathbb{R}$, which are basically functions from a state-action sequence to a real value. These definitions of the guards and the reward function give very strong possibilities to model the transition function and the generation of the reward. However, it also over-complicates the transition function and the reward calculation. In theory, they do not need the SRM if they have suitable lambda reward functions. Furthermore, they developed an algorithm that infers values for predefined holes of a fixed SRM. The user has to specify the structure of the SRM and the transitions with the holes. The learning algorithm does not learn the structure of the SRM, e.g., how many states are needed or which transitions are used. Consequently, the method cannot learn a policy end-to-end because of the required user input. This complicates the usability of the concept and reduces the user-friendliness, as the user must make inputs for each environment and task in order to generate a well-suited SRM.

Furthermore, Ardon et al. \cite{10.5555/3709347.3743526} published recently `First-Order Reward Machines' (FORMs). According to Ardon et al., FORMs are superior to RMs because they can generalize over object properties and can express conditions over whole sets of objects. Thus, FORMs often need less states than RMs. This results in smaller automatons and faster learning. However, they demonstrate the effectiveness of their method solely in small, discrete environments. In addition, FORMs also require a labeling function that maps a specific raw state to a so-called observable. Consequently, they are also not directly compatible to a standard environment. For example, if we consider a value x and want to express that this value is in a certain range, FORM requires an atom which is true if x is in the respective range. However, the specific interval has to be given a priori. It is not possible to define an interval with variables in it and FORM infers the specific values for these variables. Especially for infinite state space environments, e.g., continuous values for x, this is a big challenge because there are an infinite amount of intervals that could be used. These limitations also make FORMs a suboptimal choice for standard environments with non-Markovian reward functions. This holds for finite environments, and even more for infinite ones.

There are also more approaches to represent non-Markovian reward functions and leverage the structured information of them \cite{Corazza_2022,KR2025-55,bester2023counting}. However, most of the approaches relying on an environment with a labeling function as discussed before. Consequently, they have similar challenges and limitations as RMs, QRM and JIRP.

\section{Preliminaries}
In this section, we introduce the necessary background of Reinforcement Learning and Reward Machines.
We use a modified version of the `Office World' as our running example throughout the paper. The Office World was introduced by Icarte et al. \cite{Icarte.2018}. We define a discrete and a continuous version of the Office World. The discrete version is defined by the state space $S=\{(x,y)\ |\ x,y \in \mathbb{N}_0\ \wedge x \leq 14 \wedge y \leq 10 \}$. It is an example for a finite environment. The continuous version is defined by the state space $S = \{(x,y)\ |\ x,y \in \mathbb{R}_{\geq 0}\ \wedge x < 15 \wedge y < 11 \}$. It is an example for an infinite environment. Figure \ref{fig:Preliminaries_Office_World_Environment} shows the discrete version of the Office World.
\begin{figure}
    \centering
    \includegraphics[width=0.5\linewidth]{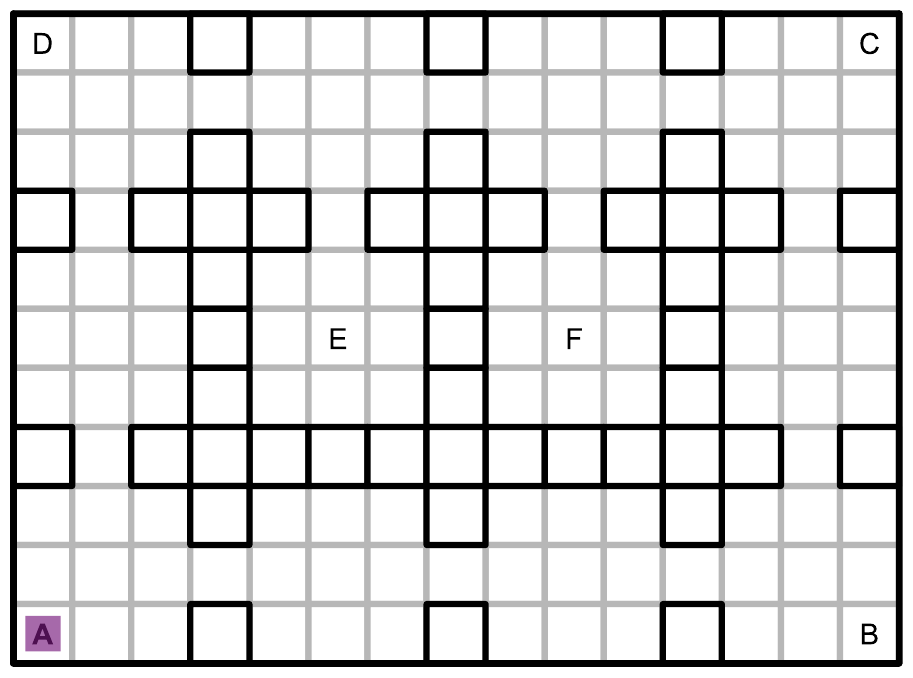}
    \caption {Office World environment. The discrete version is displayed. The labeled environments output the labels shown at the specific positions.}
    \label {fig:Preliminaries_Office_World_Environment}
\end{figure}

A Markov Decision Process (MDP) is often used to model an environment.
\begin{definition}[\textbf{MDP}]\label{definition_mdp}
A Markov Decision Process is a 5-tuple $M = (S,s_I,A,p,R)$ consisting of
a set of states $S$,
an initial state $s_I \in S$,
a finite set of actions $A$,
a probabilistic transition function $p\colon S \times A \times S \rightarrow [0,1]$, and
a reward function $R\colon (S \times A)^+ \times S \rightarrow \mathbb{R}$.
\end{definition}
Our definition of an MDP differs slightly from the classical definition in RL \cite{Puterman.1990,Kaelbling.1996,Sutton.2020}. There, the reward function is defined by $R\colon S \times A \times S \rightarrow \mathbb{R}$. Instead, we use an extended, generalized definition where rewards may depend on the current trajectory (the state-action history). Note that this generalized definition captures both Markovian and non-Markovian reward functions.

An example of a non-Markovian task in the Office World is that the agent should accomplish the following sub-tasks in the defined order:
\begin{enumerate}\label{Preliminaries_example_task}
    \item The robot has to move to the left inner office (label E);
    \item the robot has to visit the right inner office (label F); and
    \item the robot has to move back to the initial position (label A).
\end{enumerate}
The agent receives a reward of 1, 2 and 10 after it reaches the respective sub-task in the right order. The proposed RMs\footnote{Icarte et al. \cite{Icarte.2018,10.1613/jair.1.12440} introduce Reward Machines more generally. The Reward Machines, that are considered in this paper, are defined by Icarte et al. \cite{Icarte.2018,10.1613/jair.1.12440} as Simple Reward Machines.} by Icarte et al. \cite{Icarte.2018,10.1613/jair.1.12440} can model this reward structure. Figure \ref{fig:post_inner_offices_env_rm} shows a graphical representation of an RM for this task. Note that the RM uses E, F, A, which are high-level events that the environment has to output by using an explicit labeling function. Therefore, the MDP definition has to be extended to a labeled MDP. A labeled MDP consists of an MDP, propositional symbols $P$ (labels), and a labeling function $L\colon S \rightarrow 2^P$ \cite{Icarte.2018,Kaelbling.1996,Xu_Gavran_Ahmad_Majumdar_Neider_Topcu_Wu_2020}. The labeling function indicates which propositional symbols (labels) are true or false for a given environment state.

\begin{figure}
    \centering
    \begin{subfigure}[b]{0.3\linewidth}
        \includegraphics[width=\linewidth]{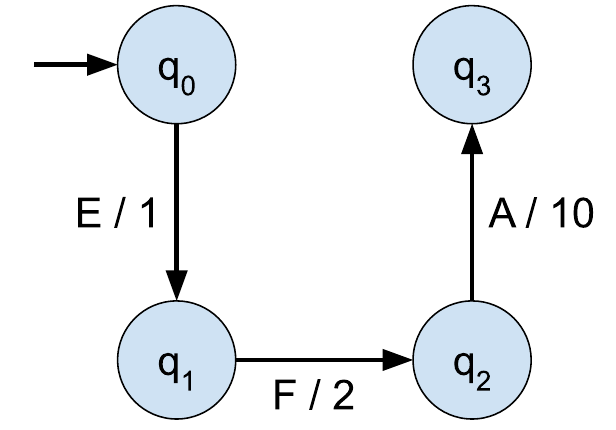}
        \subcaption{RM}
        \label{fig:post_inner_offices_env_rm}
    \end{subfigure}
    \hspace{0.05\linewidth}
    \begin{subfigure}[b]{0.45\linewidth}
        \includegraphics[width=\linewidth]{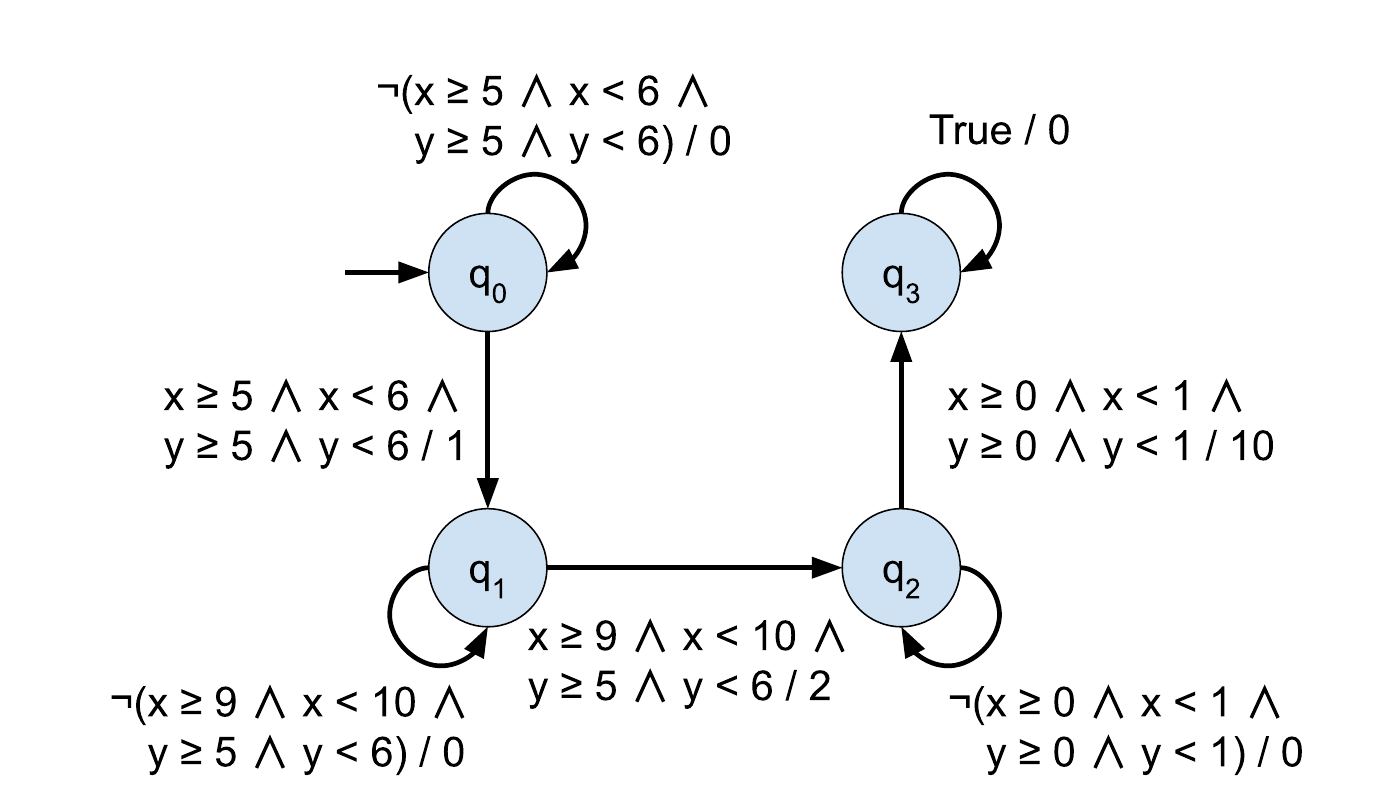}
        \subcaption{SRM}
        \label{fig:post_inner_offices_env_srm}
    \end{subfigure}
    \caption{RM and SRM for the example task in the Office World. Self-loops with an output of zero are omitted in the RM.}
    \label{fig:reward_machine_task_post_inner_offices}
\end{figure}

The agent uses a policy to select an action to interact with the environment. A deterministic policy $\pi (s)\colon S \rightarrow A$ directly maps a state to an action. A probabilistic policy $\pi (a|s)\colon S \times A \rightarrow [0,1]$ outputs a probability distribution for each state. The agent can then sample an action from this distribution.
Based on an MDP and a policy, a \textbf{trajectory} is defined by the interaction of an agent with an MDP. It is a (state, action, reward)-sequence $s_0,a_0,r_1,s_1,...,s_k,a_k,r_{k+1},s_{k+1}$ where $s_0 = s_I$. The sequence $s_0,s_1,...,s_k,s_{k+1}$ is the \textbf{state-sequence} of the trajectory, and the sequence  $r_1,r_2,...,r_{k+1}$ is the \textbf{reward-sequence} of the trajectory.

\section{Symbolic Reward Machines}\label{SRM}
In this section, we propose our Symbolic Reward Machines (SRMs) as a novel representation of non-Markovian reward functions.
SRMs integrate aspects of a Reward Machine and a symbolic automaton \cite{DAntoni.2017,10.1007/978-3-642-39274-0_3}. The major difference between RMs and SRMs is the input. RMs get propositional symbols (labels) from a specified alphabet as the input. In contrast, SRMs receive the environment state directly and handle it through guards. These guards can be defined by the specified logic.

\begin{definition}[SRM]\label{definition_srm}
A Symbolic Reward Machine is a 5-tuple $C = (L,Q,q_0,\delta, \sigma)$ and consists of
a logic $L$ with formulas $\Psi$ and an input space $S$,
a finite set of states $Q$,
an initial state $q_0 \in Q$,
a transition relation $\delta \subseteq Q \times \Psi \times Q$, and
a partial output function $\sigma \colon Q \times \Psi \rightarrow \mathbb{R}$.
We write $p \xrightarrow{\varphi} q$ for a transition $(p, \varphi, q) \in \delta$.
\end{definition}
The input space $S$ is always given by the underlying MDP. The state space of the environment can be also defined by a function $S\colon \mathrm{Var} \rightarrow U$. Each component of the state is represented by a variable $x\in \mathrm{Var}$ that is mapped to a concrete value in a universe $U$.
Furthermore, we call $\varphi$ the guard of a transition. The SRM takes the transition if the current environment state $s$ fulfills the guard. This means that $s$ is a model of $\varphi$, which is expressed by $s \models \varphi$. Additionally, the SRM returns the reward from the output function $\sigma$.
We focus solely on SRMs with LRA (Linear Real Arithmetic) \cite{10.1007/978-3-319-45641-6_26} as the logic component in this paper. This is done for simplicity. However, the concepts should be easily extendable to other logics with the following two properties. The restriction for the logic is that the validity and satisfiability problem has to be decidable.

In addition, we require the following two properties for an SRM:
\begin{enumerate}
    \item An SRM is deterministic. The transition function $\delta$ has to be deterministic. Intuitively speaking, there are no two outgoing transitions of an SRM state that can be taken at the same time. This means that the guards of any two transitions of an SRM state cannot be fulfilled at the same time. In formula terms, it is required that for all $p,q,q' \in Q$ and formulas $\varphi, \varphi' \in \Psi$ it is true that, if $p \xrightarrow{\varphi} q$, and $p \xrightarrow{\varphi'} q'$, and $q$ is not equal to $q'$, then $\varphi \wedge \varphi'$ has to be unsatisfiable.
    \item An SRM is complete. Intuitively speaking, an SRM has to handle any possible input. This means that for all SRM states $p \in Q$ and every input $x$, there has to exist a transition  $p \xrightarrow{\varphi} q$ with $q \in Q$ and $\varphi \in \Psi$ and $x$ fulfills the guard $\varphi$.
\end{enumerate}

We also define an SRM run of an SRM $C=(L, Q, q_0, \delta, \sigma)$ that can process a complete state-sequence and outputs a reward-sequence. For that, first, we define the sequence of the states of the SRM while it processes the state-sequence $s_0,...,s_n$ of an trajectory $\tau$. The sequence of SRM states $q_0,...,q_{n}$ is defined by
\[
    q_i = \left\{
    \begin{array}{ll} 
        q_0 & \text{, if } i = 0, \\
        q & \text{, with } q\in Q.\ \exists \varphi \in \Psi.\ (q_{i-1},\varphi , q) \in \delta \wedge s_{i} \models \varphi \text{ , otherwise.}
     \end{array}\right.
\]
On this basis, the output sequence $r_1,...,r_{n}$ of the SRM $C$ is defined by
\[
    r_i = \sigma(q_{i-1}, \varphi) \text{ with } \varphi \in \Psi \wedge s_{i} \models \varphi \wedge (q_{i-1}, \varphi, q_i) \in \delta.
\]
On the basis of these two definitions, we define an SRM run of the SRM $C$ on the state-sequence $s_0,...,s_n$ of a trajectory $\tau$ by $C(\tau) = r_1,...,r_{n}$.

\subsection{Learning effective policies with given SRMs}
We can represent non-Markovian reward functions with our proposed notion of SRMs. However, they are not supported by the existing learning algorithms QRM \cite{Icarte.2018} and JIRP \cite{Xu_Gavran_Ahmad_Majumdar_Neider_Topcu_Wu_2020}. Consequently, we must introduce new learning algorithms that use SRMs to learn effective policies.
We propose QSRM as a learning algorithm that supports SRMs. QSRM is similar to QRM in that it also assumes that the user provides the required SRM. However, in contrast to QRM, which requires access to an explicit labeling function, QSRM only requires the current state and the reward of the environment. Thus, QSRM respects the standard environment interaction interface in RL. The procedure of QSRM is given as pseudocode in Algorithm \ref{alg:qsrm-learning}. It can be seen that the procedure uses the identical environment interaction as in Q-Learning \cite{WatkinsChristopherJ.C.H..1992}; the environment receives an action and emits the current state and a reward. In addition, QSRM uses the same multi-update scheme as QRM to accelerate the learning process.

\begin{algorithm}
\caption{Q-Learning for SRMs (QSRM)}\label{alg:qsrm-learning}
Initialize an $\epsilon$-greedy policy $\mu$, and initialize Q-tables $q_u$ for all $u\in Q_{SRM}$ for all (state, action)-pairs with 1.\\ Then iterate:
\begin{algorithmic}[1]
    \For{t in \{0,..., number\_of\_episodes - 1\}}
        \State Start new trajectory with $s$ = $s_I$ and $u=q_0$
        \While{$s$ is not terminal state}
            \State Sample action $a$ from $\mu$ given $s$ and $u$
            \State Execute action $a$ in environment and receive next state $s'$
            \For{$u' \in Q_{SRM}$}
                    \State Get $r$, $u''$ by $SRM(u', s')$
                    \State Update $q_{u'}$ by $q_{u'}(s,a) = q_{u'}(s,a) + \alpha_t (r+\gamma \cdot \max\limits_{a'}q_{u''}(s',a')-q_{u'}(s,a))$
            \EndFor
            \State Step in SRM with input $s'$
            \State $s=s'$ and $u=$ current SRM state
        \EndWhile
    \EndFor
    \State \Return policy $\pi(s,u)=argmax_a \ q_u(s,a)$ for all $s \in S$ and $u\in Q_{SRM}$
\end{algorithmic}
\end{algorithm}

The natural question is whether our new QSRM algorithm preserves the convergence property of Q-Learning, and thus, is also guaranteed to converge to an optimal policy.

\begin{theorem}[QSRM Convergence]\label{theorem_QSRM_convergence}
QSRM always converges to an optimal policy in the limit with the same conditions as Q-Learning \cite{WatkinsChristopherJ.C.H..1992,Sutton.2020}. So, if 
\begin{enumerate}
    \item all rewards are bounded;
    \item all (state, action)-pairs are visited infinity often;
    \item learning rates $0 \leq \alpha_t < 1$; and
    \item learning rates $\alpha_t$ are decreasing at the right rate:\[\sum_t^\infty \alpha_t = \infty \text{ and } \sum_t^\infty \alpha_t^2 < \infty.\]
\end{enumerate}
\end{theorem}

\begin{proof}
We show the convergence property by building a cross product MDP $M_2=(S_2,s_{I_2}, A_2,p_2,R_2)$ with a Markovian reward function of an MDP $M=(S,s_I, A, p, R)$ where $R$ is defined by an SRM $C=(L,Q,q_0,\delta,\sigma)$. The following scheme is used to build $M_2$. The state space $S_2$ is the cross product $S \times Q$. The initial state $s_{I_2}$ is defined by the tuple $(s_I, q_0)$. The action set $A_2$ is taken from the action set $A$ from MDP $M$. The probabilistic transition function is defined by 
\[
    p_2((s,q), a, (s',q')) =
    \begin{cases}
    p(s,a, s') & \text{, if } \exists \varphi \in \Psi \text{ such that } (q,\varphi,q') \in \delta \wedge s' \models \varphi, \\
    0 & ,\text{ otherwise.}
    \end{cases}
\]
The reward function $R_2$ is defined by
\[
    R_2((s,p), a, (s',q)) = \sum_{\varphi \in \Psi} \llbracket s' \models \varphi \rrbracket \cdot \sigma(p, \varphi),
\]
whereby $\llbracket x \rrbracket$ is the Iverson bracket, which evaluates to 1 if x is true, and to 0 otherwise.
Q-Learning by Watkins et al. \cite{WatkinsChristopherJ.C.H..1992} converges for the MDP $M_2$ to the optimal policy. With the transformation procedure as defined above, QSRM does exactly the same as Q-Learning except that QSRM updates multiple states of the cross product MDP in parallel. This multi-update is applied by using the correct data of the reward function. Hence, QSRM also converges to the optimal policy with the same conditions as standard Q-Learning.\qed
\end{proof}

We also develop QSRM with function approximation methods and name it DQSRM. The core idea remains unchanged. Instead of Q-tables, neural networks are utilized to approximate the Q-function for each SRM state.

\section{Learning Symbolic Reward Machines}
QSRM provides the ability to use all existing environments out of the box, e.g., the environments of the popular Gymnasium library \cite{towers2025gymnasium}, as it respects the standard interaction scheme. However, QSRM is also not fully compatible with the standard RL definition because it assumes that the agent has access to the SRM and thus to the reward function. To overcome this limitation, we propose the LSRM algorithm that can automatically infer the SRM. LSRM can work in two settings. First, we assume that the possible guards are given by the user. We name this setting `LSRM with given formulas' (LSRM-GF).
In the second setting of LSRM, we lift this assumption. We name the second setting `LSRM with formula templates' (LSRM-FT).

The basic procedure is the same for both settings. It is given as pseudocode in Algorithm \ref{alg:LSRM}.
Like JIRP \cite{Xu_Gavran_Ahmad_Majumdar_Neider_Topcu_Wu_2020}, LSRM starts with a basic hypothesis SRM and updates it if counterexamples are found. The basic SRM consists only of one state and one self-loop. LSRM tries to learn a policy with (D)QSRM based on the current hypothesis SRM. The agent receives a reward from the environment and a reward from the current hypothesis SRM. These rewards are consistent if they are equal. If the rewards are not consistent, LSRM stops the (D)QSRM algorithm and adds the current trajectory to a set of counterexamples $\counterexamples$. We refer to the reward at timestep t of a counterexample $\counterexample \in \counterexamples$ as $r_{\counterexample,t}$, and to the observed state as $s_{e,t}$. Note that $r_{\counterexample,t}$ is only defined for $t>0$. Furthermore, we define a function $T$ that maps a counterexample $e$ to the set of all timesteps contained in that counterexample. We also need the set of timesteps excluding the last one. Therefore, we define a function $T^*$ with $T^*(\counterexample) = T(\counterexample) \setminus \{|T(\counterexample)|-1\}$. After adding the trajectory to the counterexamples, LSRM uses an SRM-inferring process to generate a new hypothesis SRM that is consistent with all counterexamples. Next, LSRM restarts the (D)QSRM algorithm with the new hypothesis SRM.

\begin{algorithm}
\caption{LSRM (Learning Symbolic Reward Machines)}\label{alg:LSRM}
    \begin{algorithmic}[1]
        \State obs = \{\}
        \State r = \{\}
        \State hypothesis\_srm = init\_basic\_srm()
        \State q\_tables = init\_q\_tables\_for\_srm(hypothesis\_srm)
        \For{number\_of\_episodes times}
            \State observations, rewards, rewards\_hypothesis\_srm = QSRM\_episode( hypothesis\_srm, q\_tables)
            \If{rewards $\neq$ rewards\_hypothesis\_srm}
                \State obs.append(observations)
                \State r.append(rewards)
                \State hypothesis\_srm = generate\_consistent\_srm(obs, r)
                \State q\_tables = init\_q\_tables\_for\_srm(hypothesis\_srm)
            \EndIf
        \EndFor
    \end{algorithmic}
\end{algorithm}

The SRM-inferring process is the core component of LSRM. It generates a consistent SRM from the set of counterexamples $\counterexamples$. The LSRM algorithm infers a consistent SRM by encoding the collected counterexamples into a constraint satisfaction problem. The encoding produces a logical formula $\varphi_{SRM}$ with two key properties: (1) the formula is satisfiable if, and only if, a consistent SRM exists, and (2) we can build a consistent SRM from a model of the formula. 
The required number of SRM states in the state set $Q$ is not known a priori. Therefore, the inferring process generates constraints for an increasing number of SRM states. At the beginning, it encodes an SRM with only one SRM state. So, it applies $Q=\{q_0\}$ where $q_0$ is w.l.o.g. the initial state. If no consistent SRM with $n$ SRM states can be found, the SRM-inferring process increases the number of SRM states to $n+1$, which results to the state set $Q=\{q_0,..., q_n\}$, and tries it again. As a result, it generates a consistent SRM with the fewest number of SRM states.

\subsection{LSRM with Given Formulas}
In LSRM with given formulas (LSRM-GF), we assume that the user provides a set of formulas $\Psi = \{\varphi_0, \dots, \varphi_{f-1}\}$ for the guards of the SRM. LSRM-GF uses four constraint types to encode the following properties of the SRM. First, the SRM has to be deterministic. To ensure this, LSRM-GF adds Constraint $\varphi_{det}$. Furthermore, the SRM is always in exactly one SRM state, and the first SRM state is always the initial state. This is modeled by Constraint $\varphi_{state}$. Moreover, the SRM has to be consistent with all counterexamples. LSRM-GF uses the constraints $\varphi_{trans}$ and $\varphi_{used}$ to model this property. The global formula $\varphi_{SRM}$ is the conjunction \[ \varphi_{SRM} = \varphi_{det} \wedge \varphi_{state} \wedge \varphi_{trans} \wedge \varphi_{used}\] of all these constraints. We will define all the components in detail in the following. First, we have to define the variables that are used to build the constraints.

\subsubsection{Variables}\label{LSRM_with_given_formulas_variables}
LSRM with given formulas uses:
\begin{itemize}
    \item Boolean variables $d_{p, \varphi , q}$ that encode whether a transition from state p to state q exists with the formula $\varphi$ as the guard;
    \item Boolean variables $x_{\counterexample, t, p}$ that encode whether the SRM is in state p after processing counterexample $\counterexample$ through the t-th observation; and
    \item Real variables $o_{p, \varphi}$ that encode the SRM output in state p and the guard $\varphi$ is fulfilled.
\end{itemize}

\subsubsection{Constraints}\label{LSRM_with_given_formulas_constraints}
On the basis of the variables described above, we define the constraints used in detail in the following.
The constraints are defined by:
\begin{align}
    &\varphi_{det} = \bigwedge_{p \in Q} \bigwedge_{q_1 \in Q} \bigwedge_{q_2 \in Q} \bigwedge_{\varphi_1 \in \Psi} \bigwedge_{\substack{\varphi_2 \in \Psi,\\ (p,\varphi_1,q_1) \\ \neq \\ (p,\varphi_2,q_2)}} d_{p, \varphi_1 , q_1} \wedge d_{p, \varphi_2 ,q_2} \Rightarrow \mathrm{Unsat} (\varphi_1 \wedge \varphi_2) \label{LSRM_with_given_formulas_constraint_det}\\[1em]
    &\varphi_{state} = \bigwedge_{\counterexample \in \counterexamples} x_{\counterexample,0,q_0} \bigwedge_{t \in T(\counterexample)} (\bigvee_{p\in Q} x_{\counterexample,t,p}) \wedge (\bigwedge_{p\in Q} \bigwedge_{\substack{q \in Q \\ p \neq q}} \neg (x_{\counterexample,t,p} \wedge x_{\counterexample,t,q})) \label{LSRM_with_given_formulas_constraint_state}\\[1em]
    &\varphi_{trans} = \bigwedge_{\counterexample \in \counterexamples} \bigwedge_{t \in T^*(\counterexample)} \bigwedge_{p \in Q} \bigwedge_{\varphi \in \Psi} \bigwedge_{q \in Q} x_{\counterexample,t,p} \wedge d_{p, \varphi , q} \wedge (s_{\counterexample,t+1}\models \varphi) \Rightarrow x_{\counterexample, t+1, q} \wedge o_{p, \varphi} = r_{\counterexample, t+1} \label{LSRM_with_given_formulas_constraint_trans}\\[1em]
    &\varphi_{used} = \bigwedge_{\counterexample \in \counterexamples} \bigwedge_{t \in T^*(\counterexample)} \bigvee_{p \in Q} \bigvee_{\varphi \in \Psi} \bigvee_{q \in Q} x_{\counterexample,t,p} \wedge d_{p, \varphi , q} \wedge (s_{\counterexample,t+1}\models \varphi). \label{LSRM_with_given_formulas_constraint_used}
\end{align}

The four constraints are used to encode in general two properties. Constraint \ref{LSRM_with_given_formulas_constraint_det} guarantees that the SRM is deterministic and the constraints \ref{LSRM_with_given_formulas_constraint_state}, \ref{LSRM_with_given_formulas_constraint_trans} and \ref{LSRM_with_given_formulas_constraint_used} encode the runs over all counterexamples, which leads to the property that the SRM is consistent with all these counterexamples.

The basic idea of Constraint \ref{LSRM_with_given_formulas_constraint_det} to encode the determinism property is the following. If there are any two different outgoing transitions in one SRM state that should be used, then the conjunction of both transition guards must be unsatisfiable. The usage of two outgoing transitions would be indicated by setting variables $d_{p, \varphi_1 , q_1}$ and $d_{p, \varphi_2 ,q_2}$ to true. Note that the Unsat term can be precalculated as the formulas are given a priori. This precalculation is essential as the use of a universal quantifier would be required otherwise. This universal quantifier would lead to a significant decrease of the SMT solver performance.

Constraint \ref{LSRM_with_given_formulas_constraint_state} encodes that the SRM is always in exactly one SRM state. In addition, naturally, the SRM must always start in the initial SRM state $q_0$. The variables $x_{e,0,q_0}$ ensure that the SRM starts in the initial state for all counterexamples. The remaining part of the constraint models the following two properties. First, the SRM is always in at least one state. That is modeled by the discjunction over all variables $x_{e,t,p}$ for that counterexample. Second, the SRM cannot be in two states at the same time. That is modeled by the negation $\neg (x_{e,t,p} \wedge x_{e,t,q})$.

Constraint \ref{LSRM_with_given_formulas_constraint_trans} encodes consistency across all counterexamples.
For each state in the state-sequence of a counterexample, the SRM has to output the observed reward of the reward-sequence.
The SRM is in one specific state $p$ at timestep $t$ during the processing of counterexample $e$. If there is an outgoing transition with the guard $\varphi$ to state $q$ and the observed environment state of the counterexample is a model of this guard, the SRM takes this transition. This means that the SRM is in state $q$ in the next timestep. In addition, the generated reward of the SRM, which is encoded by $o_{p,\varphi}$, has to be equal to the observed reward in the counterexample $\counterexample$.

Constraint \ref{LSRM_with_given_formulas_constraint_used} guarantees that for each step in the counterexample, the SRM has to be in one state $p$, there is a transition from a state $p$ to a state $q$ with the guard $\varphi$, and the observation of the step in the counterexample is a model of the guard $\varphi$. In other words, we need a step that the SRM takes. So, it guarantees that one premise of the implications in Constraint \ref{LSRM_with_given_formulas_constraint_trans} is valid.

All these constraints guarantee that a consistent SRM can be generated by finding a model for the global formula $\varphi_{SRM}$. Next, we describe how LSRM-GF generates a consistent SRM given a model of $\varphi_{SRM}$.

\subsubsection{Generate SRM from Model}\label{LSRM_GF_Generate_SRM_from_Model}
The SMT solver returns a model of $\varphi_{SRM}$, if and only if, $\varphi_{SRM}$ is satisfiable. We assume w.l.o.g. that the SRM states form the set $Q=\{q_0,q_1,...,q_n\}$ where $q_0$ denotes the initial state. LSRM-GF can derive a consistent SRM $C = (L,Q,q_0,\delta, \sigma)$ from the model. The logic $L$ is set to LRA (see Section \ref{SRM}). The output function $\sigma$ is defined by $\sigma(p, \varphi) = o_{p, \varphi}$ for all $p \in Q$ and $\varphi \in \Psi$.
To generate the transition relation $\delta$, LSRM-GF adds $(p,\varphi , q)$ to $\delta$ for all $p,q \in Q$ and $\varphi \in \Psi$ if $d_{p, \varphi, q}$ is set to true in the model.

With this procedure, LSRM-GF constructs a consistent SRM based on a model for $\varphi_{SRM}$. However, the resulting SRM is not guaranteed to be complete. It is possible that for some inputs, the learned SRM does not have a transition with a fulfilled guard. Thus, it is not a valid SRM for our definition in Section \ref{SRM}.
To achieve this property while preserving quantifier freedom in the constraints, we first invoke the SMT solver without further constraints and then add a post-processing step afterwards. We process each SRM state $p\in Q$ with the following procedure. We define a formula $\varphi_{p,out}$ that is the negated disjunction of all outgoing guards. Therefore, first, we define the set of formulas $X^p = \{\varphi | \varphi \in \Psi \wedge \exists q\in Q. \ d_{p,\varphi,q} = True\}$. The set $X^p$ includes all formulas that are used as a guard for the SRM state $p$, which is indicated by a variable $d_{p,\varphi,q}$ is set to true. Based on the formula set $X^p$, the formula $\varphi_{p,out}$ is defined by \[\varphi_{p,out} = \neg \bigvee_{\varphi \in X^p} \varphi.\] If $\varphi_{p,out}$ is satisfiable, we add a self-loop transition with $\varphi_{p,out}$ as the guard and an output of zero. This is done by extending the set $\Psi$ of the generated SRM with $\varphi_{p,out}$, adding $(p,\varphi_{p,out},p)$ to $\delta$ and letting $\sigma(p, \varphi_{p,out})$ maps to 0. This procedure works because we only want to generate a consistent SRM. No counterexample uses the transition that we add because constraints \ref{LSRM_with_given_formulas_constraint_trans} and \ref{LSRM_with_given_formulas_constraint_used} would force a transition that covers the example, otherwise. As a result of this post-processing step, LSRM-GF generates a fully working SRM that is consistent with all counterexamples.

\subsection{LSRM with Formula Templates}
LSRM-GF still has the limitation that the user must provide the concrete formula set $\Psi$. This assumption is not feasible for many practical tasks because the user must have an extensive domain expertise about the environment and the task. To address this limitation, we lift this assumption in the second setting that we name LSRM with formula templates (LSRM-FT). LSRM-FT can learn an SRM without any further information except the standard RL environment output, which consists of an environment state and a reward. This makes LSRM-FT compatible with all existing environments and the standard RL definitions. LSRM-FT automatically infers for each SRM state the required formulas for the guards. This means that LSRM-FT maintains a formula set $\Psi^p$ for each SRM state $p\in Q$. Let $f \in \mathbb{N}$ denote the number of formulas that can be used as guards for an SRM state. There are two options to handle the number of formulas $f$. The first one is to increase the number of states and the number of formulas in parallel in order to search for a model. The second option is that LSRM-FT handles the number of formulas per state as a hyperparameter. This option is computationally faster and is guaranteed to find the SRM with the fewest number of SRM states for the specified number of formulas. In the evaluation, we will use the second option. By using $f$ formulas per SRM state, $\Psi^p$ is defined by $\Psi^p = \{\varphi_{0}^p, ..., \varphi_{f-1}^p\}$. As the algorithm's name suggests, LSRM-FT uses formula templates as fixed building blocks for the formulas in $\Psi^p$. Formula templates are formulas that contain free variables besides those representing the inputs of the SRM. As an example, one formula template could be $\varphi = x \geq b_1 \wedge x < b_2$. This formula template could be used to find an interval for a component $x$. LSRM-FT infers suitable values for the additional free variables; in the given example, for $b_1$ and $b_2$. This means that LSRM-FT concretizes a formula template. A concretized version of the example formula template could be $\varphi = x \geq 2 \wedge x < 3$. Inferring formulas from templates is a significant improvement to overcome the limitation that the user has to provide the concrete formulas, and thus, information about the reward structure. In principal, every formula can be used as a formula template.

The basic idea of LSRM-FT is the same as for LSRM-GF. LSRM-FT uses constraints to guarantee the same properties as LSRM-GF. The SRM is deterministic, the SRM is always in exactly one SRM state, and the SRM is consistent with all counterexamples. Therefore, the constraints $\varphi_{det}$, $\varphi_{state}$, $\varphi_{trans}$, and $\varphi_{used}$ are adapted from LSRM-GF to LSRM-FT.

In the following, first, we define the needed variables again. Afterwards, we define and describe the required adjustments of the constraints.

\subsubsection{Variables}
The general idea of encoding a consistent SRM is the same as in LSRM with given formulas. Hence, LSRM with formula templates requires similar variables for the structure and the runs over the counterexamples. We adapt the variables $d_{p, \varphi, q}$, $x_{e,t,p}$, and $o_{p,\varphi}$ as defined for LSRM-GF in Section \ref{LSRM_with_given_formulas_variables}. 
So, LSRM-FT uses:
\begin{itemize}
    \item Boolean variables $d_{p, \varphi^p , q}$ that encode whether a transition from state p to state q exists with the formula $\varphi^p$ as the guard;
    \item Boolean variables $x_{\counterexample, t, p}$ that encode whether the SRM is in state p after processing counterexample $\counterexample$ through the t-th observation; and
    \item Real variables $o_{p, \varphi^p}$ that encode the SRM output in state p and the guard $\varphi^p$ is fulfilled.
\end{itemize}

In addition to these variables, LSRM-FT uses further variables for the formula templates that are used to build the guards. The required variables depend on the specific formula template that is used. We will use a box formula template in the evaluation of our methods. A box formula template can describe an interval for each component of the environment state. We described in Section \ref{SRM} that the state space can also be defined by a function $S \colon \mathrm{Var} \rightarrow U$. On the basis of this definition, the box formula template specifies an interval for each $x \in \mathrm{Var}$. We also allow a negated box formula template, whereby the complete specified box is negated. Consequently, the box formula template is defined by \[\varphi_{box} = pos_{box} \Leftrightarrow \bigwedge_{x \in \mathrm{Var}} x \geq b_{x_\geq} \wedge x < b_{x_<},\] where $pos_{box}$ indicates whether it is a box formula template or a negated box formula template. This formula template is used in all SRM states and guards. The following variables are required to define the formula templates for the guards:

\begin{itemize}
    \item Boolean variables $pos^p_i$ that encode whether the formula template of formula $\varphi^p_i$ of state $p$ appears in positive or negated form; and
    \item Real variables $b^p_{i, x_j}$ that encode the right sides of the inequality equations of the box template for formula $\varphi^p_i$ and the environment state component $x\in \mathrm{Var}$ and the operator $j \in \{ \geq, <\}$.
\end{itemize}

On the basis of these variables and our general concept of box templates, we can define the formula templates for each SRM state $p$. The set of formula templates for a state $p$ is $\Gamma^p = \{\psi^p_0, ..., \psi^p_{f-1}\}$, where each formula template is defined by
\[ \psi^p_i = pos^p_i \Leftrightarrow \bigwedge_{x \in \mathrm{Var}} x \geq b^p_{x_\geq} \wedge x < b^p_{x_<}.\]
The number of formula templates associated with each SRM state is given by $|\Gamma^p|$ and corresponds to the number of formulas $f$ that LSRM-FT can use in each SRM state.
In the naive approach, each guard is constructed directly from its formula template by setting $\varphi^p_i = \psi^p_i$.

\subsubsection{Constraints}
The constraints of LSRM-FT are adapted from the constraints of LSRM-GF in Section \ref{LSRM_with_given_formulas_constraints}. The modified constraints are:
\begin{align}
    &\varphi_{det} = \bigwedge_{p \in Q} \bigwedge_{q_1 \in Q} \bigwedge_{q_2 \in Q} \bigwedge_{\varphi_1 \in \Psi^p} \bigwedge_{\substack{\varphi_2 \in \Psi^p,\\(p,\varphi_1,q_1)\\\neq\\(p,\varphi_2,q_2)}} d_{p, \varphi_1 , q_1} \wedge d_{p, \varphi_2 ,q_2} \Rightarrow \mathrm{Unsat} (\varphi_1 \wedge \varphi_2) \label{LSRM_Fromula_Templates_Constraint_det}\\[1em]
    &\varphi_{state} = \bigwedge_{\counterexample \in \counterexamples} x_{\counterexample,0,q_0} \bigwedge_{t \in T(\counterexample)} (\bigvee_{p\in Q} x_{\counterexample,t,p}) \wedge (\bigwedge_{p\in Q} \bigwedge_{\substack{q \in Q \\ p \neq q}} \neg (x_{\counterexample,t,p} \wedge x_{\counterexample,t,q})) \label{LSRM_Formula_Templates_Constraint_state}\\[1em]
    &\varphi_{trans} = \bigwedge_{\counterexample \in \counterexamples} \bigwedge_{t \in T^*(\counterexample)} \bigwedge_{p \in Q} \bigwedge_{\varphi \in \Psi^p} \bigwedge_{q \in Q} x_{\counterexample,t,p} \wedge d_{p, \varphi , q} \wedge \mathrm{substitute}(\varphi, s_{\counterexample,t+1}) \Rightarrow x_{\counterexample, t+1, q} \wedge o_{p, \varphi} = r_{\counterexample, t+1} \label{LSRM_Formula_Templates_Constraint_trans}\\[1em]
    &\varphi_{used} = \bigwedge_{\counterexample \in \counterexamples} \bigwedge_{t \in T^*(\counterexample)} \bigvee_{p \in Q} \bigvee_{\varphi \in \Psi^p} \bigvee_{q \in Q} x_{\counterexample,t,p} \wedge d_{p, \varphi , q} \wedge \mathrm{substitute}(\varphi, s_{\counterexample,t+1}). \label{LSRM_Formula_Templates_Constraint_used}
\end{align}

There are two differences in the constraints. First, we replace each $\Psi$ with $\Psi^p$ because LSRM-FT uses a separate formula set per SRM state. And second, we cannot precalculate $s_{\counterexample, t+1} \models \varphi$. Instead, we use the function $\mathrm{substitute}$ that performs substitution of the variables $x \in \mathrm{Var}$ through the observed environment state $s_{\counterexample, t+1}$. However, one challenge remains. In Section \ref{LSRM_with_given_formulas_constraints}, we described that we can precalculate $\mathrm{Unsat}(\varphi_1 \wedge \varphi_2)$, and thus, eliminate a large part of the work for the SMT solver. However, this approach will not work for LSRM-FT. The formulas $\varphi_1$ and $\varphi_2$ are now formula templates and LSRM-FT should infer values for these templates. So, we cannot precalculate the unsatisfiability anymore. The naive approach would be to add a universal quantifier. Instead of $\mathrm{Unsat}(\varphi_1 \wedge \varphi_2)$, it would use $\forall s\in S.\ \neg (\mathrm{substitute}(\varphi_1, s) \wedge \mathrm{substitute}(\varphi_2, s))$. However, this would lead to a major degeneration in the performance of the SMT solver.

To overcome this performance issue, we propose another approach to guarantee determinism without introducing quantifiers. As a result, our version retains the quantifier freedom. Currently, each guard formula $\varphi^p_i \in \Psi^p$ of an SRM state $p$ is defined by its formula template $\psi^p_i \in \Gamma^p$. So, $\varphi^p_i = \psi^p_i$ applies. However, as we mentioned, it is possible that two guard formulas $\varphi^p_1, \varphi^p_2 \in \Psi^p$ are satisfiable at the same time ($\varphi^p_1 \wedge \varphi^p_2$ is satisfiable). Hence, we have to adjust the guard formulas that we can guarantee that two of them are not satisfiable at the same time. We reach this property by using the following mechanism. Instead of the guard $\varphi^p_i$ only consists of the formula template $\psi^p_i$, it consists of all formula templates in $\Gamma^p$. The idea is that the related formula template $\psi^p_i$ should be fulfilled and all other formula templates $\psi^p_j$ with $i \neq j$ are not fulfilled. Consequently, the guard formula $\varphi^p_i \in \Psi^p$ is defined by
\[\varphi^p_i = \psi^p_i \bigwedge_{\substack{\psi^p_j \in \Psi^p \\i\neq j}} \neg \psi^p_j.\]
This definition always guarantees determinism. Let us take any two guard formulas $\varphi^p_1 , \varphi^p_2 \in \Psi^p$ with $\varphi^p_1 \neq \varphi^p_2$. To fulfill $\varphi^p_1$, $\psi^p_1$ has to be fulfilled. However, $\psi^p_1$ appears in negated form in $\varphi^p_2$. As a result, $\varphi^p_2$ cannot be fulfilled. This also applies vice versa and for each guard formula pair and for each SRM state. The remaining part is to force that there are no two outgoing transitions that uses the same guard because it would be non deterministic otherwise. Therefore, we rewrite the determinism constraint $\varphi_{det}$ to \[\varphi_{det} = \bigwedge_{p \in Q} \bigwedge_{q_1 \in Q} \bigwedge_{\substack{q_2 \in Q \\ q_1 \neq q_2}} \bigwedge_{\varphi \in \Psi^p} \neg(d_{p, \varphi , q_1} \wedge d_{p, \varphi ,q_2}) \label{LSRM_Fromula_Templates_Constraint_new_det}.\] Consequently, the determinism property of the SRM is always guaranteed.
The global formula $\varphi_{SRM}$ remains as before:\[\varphi_{SRM} = \varphi_{det} \wedge \varphi_{state} \wedge \varphi_{trans} \wedge \varphi_{used}.\]

 \subsubsection{Generate SRM from Model}
 LSRM-FT can generate a consistent SRM from a given model for the global formula $\varphi_{SRM}$. First, it sets the inferred variables for the formula templates $\Gamma^p$ for each SRM state p. The remaining procedure is similar as the procedure described in Section \ref{LSRM_GF_Generate_SRM_from_Model} for LSRM-GF. The only difference is that we replace every $\Psi$ with a $\Psi^p$ because LSRM-FT relies on a formula set per SRM state.

\subsection{LSRM Convergence}\label{LSRM_Convergence}
In this section, we analyze the convergence properties of our two LSRM methods. Therefore, we analyse two convergence properties for each method in the following. First, whether LSRM almost surely learns an equivalent SRM, and second, whether LSRM learns an optimal policy. We say that two SRMs $C_1$ and $C_2$ are almost surely equivalent for an MDP $M$, noted with $C_1 \simeq C_2$, if for all trajectories $\tau$ with a non negative probability to occur, the generated reward-sequences $C_1(\tau)$ and $C_2(\tau)$ are equal. In particular, the probability of the set of all trajectories with $C_1(\tau) \neq C_2(\tau)$ is zero.

The first result is that all our methods which rely on DQSRM do not converge to an optimal policy. The underlying method is based on all three elements (function approximation, bootstrapping, and off-policy training) of the deadly triad \cite{Sutton.2020}.

\subsubsection{Convergence of LSRM with Given Formulas}
In the following, we describe the convergence properties of LSRM-GF. Therefore, first, we show that the learned SRM is almost surely equivalent to the one used in the environment. Afterwards, we argue that LSRM converges to an optimal policy.

\begin{theorem}[LSRM-GF convergence to equivalent SRM]\label{Theorem_LSRM_Given_Formulas_Convergence_Equivalent_SRM}
LSRM with given formulas converges to an almost surely equivalent SRM in the limit if all state-sequences are observed infinitely often and the given formulas are sufficient to build an equivalent SRM.
\end{theorem}

\begin{proof}
We prove the theorem by contradiction. Assume that the SRM $C_1$ used by LSRM and the SRM $C_2$ used in the environment are not almost surely equivalent, and let trajectory $\tau$ be the evidence for that. This means that $C_1(\tau) \neq C_2(\tau)$, and the probability of observing $\tau$ is greater than zero. The RL setup guarantees that such a trajectory $\tau$ will be observed in the limit. Furthermore, let w.l.o.g. epoch $m$ be the first epoch in which $\tau$ occurs. At this point, LSRM stops the learning process and adds $\tau$ to the set of counterexamples $\counterexamples$. Subsequently, LSRM infers a new SRM that is consistent with all counterexamples and therefore is also consistent with $\tau$. Hence, from episode $m+1$ onward, the SRM is consistent for $\tau$ with the SRM used in the environment. As a result, the SRM inferred by LSRM is almost surely equivalent to the SRM used in the environment in the limit. 

The remaining part to show is that LSRM-GF guarantees to generate a consistent SRM from the counterexamples $\counterexamples$. Therefore, we must show that if a consistent SRM exists, then there also exists a model for the formula $\varphi_{SRM}$ and vice versa. We start by proving the forward implication.
This means that a consistent SRM $C=(L, Q, q_0, \delta, \sigma)$ is given and we have to show that $\varphi_{SRM} = \varphi_{det} \wedge \varphi_{state} \wedge \varphi_{trans} \wedge \varphi_{used}$ is satisfiable by constructing a model for it. Since the formula is a conjunction of multiple subformulas, it is sufficient to show that all subformulas are fulfilled by a model.

First, we consider subformula $\varphi_{det}$ and show its satisfiability. We set all variables $d_{p,\varphi,q}$ to true if $(p,\varphi,q) \in \delta$ and to false otherwise. These assignments also guarantee that all $d_{p,\varphi,q}$ are set. Since the SRM is deterministic, it holds that if there exist two distinct transitions $(p,\varphi_1,q_1), (p,\varphi_2,q_2) \in \delta$, then $\varphi_1$ and $\varphi_2$ are not satisfiable at the same time. This also means that if any two different $d_{p,\varphi_1,q_1}, d_{p,\varphi_2,q_2}$ variables are set to true, then the conjunction of $\varphi_1$ and $\varphi_2$ is unsatisfiable. This results in the formula $\varphi_{det}$ evaluating to true in the case that two $d_{p,\varphi,q}$ variables are true. Furthermore, if at least one of any two variables $d_{p,\varphi_1,q_1}, d_{p,\varphi_2,q_2}$ is false, then the premise is false and thus the implication is true. This results in the property that $\varphi_{det}$ evaluates to true in all cases.

The second subformula is $\varphi_{state}$. To show that $\varphi_{state}$ is satisfiable, let us pick an arbitrary counterexample $\counterexample \in \counterexamples$. In addition, let $n$ denote the length of the state-sequence of the counterexample. As we described in Section \ref{SRM}, the run of the SRM $C$ induces a sequence of SRM states $q_0,...,q_n$. We set $x_{\counterexample, t,p}$ to true if $p = q_t$ and to false otherwise. Since $q_0$ is always the first SRM state, as it is the initial state of the SRM, the first part of the formula $\varphi_{state}$ is true. Furthermore, the SRM is in every timestep exactly in one SRM state. At timestep $t$, the SRM is in state $q_t$. This also means that for each $t$, we set only one variable $x_{\counterexample, t, p}$ for counterexample $\counterexample$ to true and for all other SRM states to false. Consequently, the second part of the formula, which describes that only one $x_{\counterexample, t, p}$ can be true at each timestep for a counterexample, evaluates to true under the assignments we made. We make these assignments for each counterexample and thus the complete formula $\varphi_{state}$ evaluates to true.

The remaining subformulas are $\varphi_{trans}$ and $\varphi_{used}$. We already set all variables $d_{p,\varphi,q}$ and $x_{\counterexample, t, p}$. They are fixed and cannot be changed in our proof anymore. We consider any counterexample $\counterexample \in \counterexamples$ again. This counterexample consists of a state-sequence $s_0,...,s_n$ and a reward sequence $r_1,...,r_n$. The run of the SRM $C$ induces a sequence of SRM states $q_0,...,q_n$ again. First, we show that for each timestep $t$, $\varphi_{used}$ is true. The SRM is in state $q_t$ in timestep $t$. In $\varphi_{state}$, we set the variable $x_{e,t,q_t}$ to true. Furthermore, the SRM transitions to state $q_{t+1}$ by using a transition with a guard $\varphi_t$ and the observed state $s_{e,t+1}$, which is the input of the SRM, fulfills this guard. We set in $\varphi_{det}$ all $d_{p,\varphi,q}$ to true if and only if the SRM has a transition from state p to state q with guard $\varphi$. This means that for the current timestep $t$, we set the variable $d_{q_t,\varphi_t,q_{t+1}}$ to true in $\varphi_{det}$. In addition, the observed state $s_{e,t+1}$ fulfills the guard. Hence, $s_{e,t+1} \models \varphi_t$ also holds. Consequently, $\varphi_{used}$ entirely evaluates to true. 

To prove that $\varphi_{trans}$ evaluates to true, the remaining part to show is that only one premise is true and for the true premise, the conclusion is correct. We already argued that for each counterexample $\counterexample$ and timestep $t$, the SRM is in exactly one SRM state $q_t$. This means that for all other states the premise is wrong because for all other states, we set the variable $x_{e,t,p}$ to false. In addition, the SRM is deterministic. This means that there is only one outgoing transition that is fulfilled by the observed state. So, either we set the variable $d_{p,\varphi,q}$ to false in $\varphi_{det}$ or the observed state $s_{e,t+1}$ is not a model of $\varphi$ and thus $s_{e,t+1} \models \varphi$ evaluates to false. As a result, there is exactly one premise in $\varphi_{trans}$ that is true for each counterexample $\counterexample$ and timestep $t$. For this valid premise, we have to prove that the conclusion is correct. The SRM is for counterexample $\counterexample$ and timestep $t$ in state $q_t$. In addition, the SRM used the transition $(q_t, \varphi_t, q_{t+1})$ and thus the observed state $s_{\counterexample, t+1}$ fulfills $\varphi_t$. Moreover, the SRM outputs the reward $r_{\counterexample,t}$. Consequently, the SRM is in state $q_{t+1}$ in timestep $t+1$. We set the variable $x_{e, t+1, q_{t+1}}$ to true in $\varphi_{state}$. The remaining part to show is that $o_{q_t,\varphi_t} = r_{\counterexample, t+1}$ evaluates to true. We know that the SRM always outputs for the transition it took the same reward that is the observed reward $r_{\counterexample,t+1}$. Thus, we can set variable $o_{q_t,\varphi_t}$ to the reward $r_{\counterexample,t+1}$ and thus the complete conclusion evaluates to true. We never try to set $o_{q_t,\varphi_t}$ to two different values because then the SRM has to output two different values for the same SRM state and guard. This is not possible because one formula can only be used one time in an SRM state as a guard. Otherwise, the SRM is not deterministic anymore. As a result, the complete formula $\varphi_{trans}$ evaluates to true for our assignments.

The remaining variables, which are only variables $o_{p,\varphi}$, can be set to zero in order to obtain a complete model where every variable has an assigned value. Since all subformulas are fulfilled by our constructed model, the global formula $\varphi_{SRM}$ evaluates to true. This shows that whenever a consistent SRM exists, there also exists a model satisfying $\varphi_{SRM}$.

Next, we show that if there exists a model $m$ of the formula $\varphi_{SRM}$, then there exists a consistent SRM.
We have to show the following four properties:
\begin{enumerate}
    \item SRM is deterministic;
    \item SRM always starts in $q_0$;
    \item SRM is consistent with all $\counterexample \in \counterexamples$; and
    \item SRM is complete: This is directly given by the argumentation of the post-processing step in Section \ref{LSRM_GF_Generate_SRM_from_Model}.
\end{enumerate}
The SRM is constructed as described in Section \ref{LSRM_GF_Generate_SRM_from_Model}.

First, we start to show that the constructed SRM is deterministic. We only add transitions $(p,\varphi,q)$ to $\delta$ if $d_{p,\varphi,q}$ is set to true. Since $\varphi_{det}$ evaluates to true, the conjunction of the guards of any two outgoing transitions in an SRM state is unsatisfiable. Thus, the generated SRM is always deterministic.

The second property that we have to prove is that the SRM always starts in the initial state $q_0$. In the model $m$, every variable $x_{\counterexample,0,q_0}$ is set to true. According our construction of an SRM, we set this state to the initial state of the SRM, and thus, the SRM always starts in the initial state.

The most important property to prove is the third one, which claims that the SRM is consistent with all counterexamples. Let us pick any counterexample $\counterexample \in \counterexamples$ that consists of the state-sequence $s_0,...,s_n$ and the reward-sequence $r_1,...,r_n$. Next, we have to show that the SRM is consistent with this counterexample. The formula $\varphi_{state}$ guarantees that the SRM is always in exactly one SRM state at every timestep $t$ and that the first SRM state is the initial state. It is important that the generated SRM always transitions to the next correct state that is specified with the variable $x_{\counterexample, t, p}$ and that the SRM outputs the observed return. Let us assume that the SRM is in state $p$ after processing the counterexamples $\counterexample$ trough the t-th observation. We have to show that the SRM transitions to the next state that is indicated as the variable $x_{\counterexample, t+1, q}$ is set to true in the model. The formula $\varphi_{used}$ guarantees that there is a variable $d_{p,\varphi,q}$ set to true and the observed state fulfills the guard $\varphi$. As our construction, we add a transition $(p,\varphi,q)$ to $\delta$. Since we already showed that the SRM is deterministic, the SRM has exactly one transition that it can take. Furthermore, the next SRM state is $q$, which corresponds to the variable $x_{\counterexample, t+1, q}$ is set to true in the model. The last thing to check is whether the SRM outputs the correct reward $r_{\counterexample,t+1}$. The formula $\varphi_{trans}$ guarantees that the model variable $o_{p,\varphi}$ is set to $r_{\counterexample, t+1}$. In the construction of the SRM, the output function $\sigma$ is defined by $\sigma(p, \varphi) = o_{p, \varphi}$ for all $p \in Q$ and $\varphi \in \Psi$. Consequently, the SRM outputs the correct reward $r_{\counterexample,t+1}$. As a result, if there exists a model $m$ of the formula $\varphi_{SRM}$, then there exists a consistent SRM.

We showed that whenever a consistent SRM exists, there also exists a model satisfying $\varphi_{SRM}$, and whenever a model $m$ of the formula $\varphi_{SRM}$ exists, then there exists a consistent SRM. Consequently, LSRM-GF guarantees to generate a consistent SRM from the counterexamples $\counterexamples$.\qed
\end{proof}

\begin{corollary}[LSRM-GF convergence to optimal policy]
LSRM with given formulas converges to an optimal policy in the limit for finite state spaces with the same conditions as QSRM (see Theorem \ref{theorem_QSRM_convergence}) and LSRM-GF converges to an equivalent SRM.
\end{corollary}

\begin{proof}
    We prove above that LSRM with given formulas converges to an almost surely equivalent SRM. The convergence of QSRM then guarantees that LSRM-GF also converges to an optimal policy in the limit.
\end{proof}

\subsubsection{Convergence of LSRM with Formula Templates}
The natural question is whether LSRM-FT has the same convergence properties as LSRM-GF. In the following, we argue that this is the case.
\begin{theorem}[LSRM-FT convergence to equivalent SRM]
LSRM with formula templates converges to an almost surely equivalent SRM in the limit if all state-sequences are observed infinitely often and the formula templates used are capable of expressing the formulas that are used as guards at the transitions of the underlying SRM in the environment.
\end{theorem}

\begin{proof}
The proof is adapted from Theorem \ref{Theorem_LSRM_Given_Formulas_Convergence_Equivalent_SRM}. Only the SRM-inferring process is different. However, both SRM-inferring processes generate a new SRM that is consistent with the counterexamples. The proof that LSRM-FT generates a consistent SRM from the counterexamples $\counterexamples$ is analogous to the proof for LSRM-GF. The difference is that for LSRM-FT the determinism of different guards is given through the construction of them and it is only necessary to check that a guard is not used for two outgoing transitions in one SRM state. Furthermore, each $\Psi$ has to be replaced by an $\Psi^p$. As a result, LSRM with formula templates converges to an almost surely equivalent SRM.
\end{proof}

\begin{corollary}[LSRM-FT convergence to optimal policy]
LSRM with formula templates converges to an optimal policy in the limit for finite state spaces with the same conditions as QSRM (see Theorem \ref{theorem_QSRM_convergence}) and LSRM-FT converges to an equivalent SRM.
\end{corollary}

\begin{proof}
    We prove above that LSRM with formula templates converges to an almost surely equivalent SRM. The convergence of QSRM then guarantees that LSRM-FT also converges to an optimal policy in the limit.
\end{proof}

\section{Experimental Results}
To demonstrate the effectiveness of our methods, we implemented a prototype in Python using PyTorch \cite{NEURIPS2019_bdbca288} for training the neural networks. We evaluate the methods in discrete (finite) and continuous (infinite) environments. The SMT solver used is Z3 \cite{Moura.2008}, and the experiments were carried out on an Intel Core i9-13900K and an NVIDIA RTX A6000.

\subsubsection{Environments and Tasks} 
Here, we describe the environments and tasks that we use to evaluate the methods. We evaluate the methods in two types of environments.

The first environment type is the discrete and continuous Office World. We perform two tasks on both variants of the Office World environment. The first task is the example task that we introduced in Section \ref{Preliminaries_example_task}. The RM and SRM for the task are visualized in Figure \ref{fig:reward_machine_task_post_inner_offices}. We refer to the task as `post\_inner\_offices'. The second task is named `diagonal\_run'. The goal of the task is that the robot navigates from the initial state to the upper right corner (label C), then to the upper left corner (label D), and finally to the bottom right corner (label B). The RM and SRM for this task are visualized in Figure \ref{fig:task_diagonal_run_machines}.
\begin{figure}
    \centering
    \begin{subfigure}[b]{0.3\linewidth}
        \includegraphics[width=\linewidth]{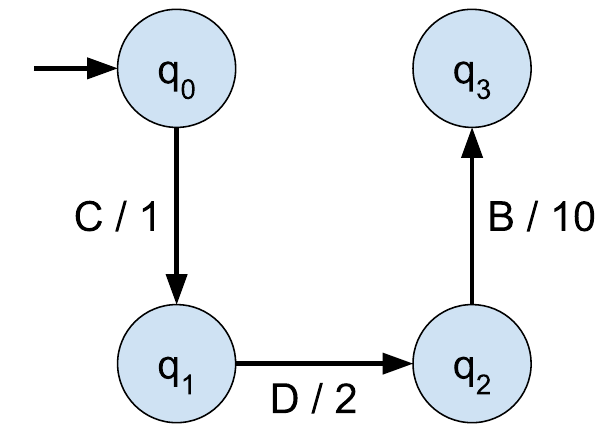}
        \caption{RM}
    \end{subfigure}
    \hspace{0.07\linewidth}
    \begin{subfigure}[b]{0.45\linewidth}
         \includegraphics[width=\linewidth]{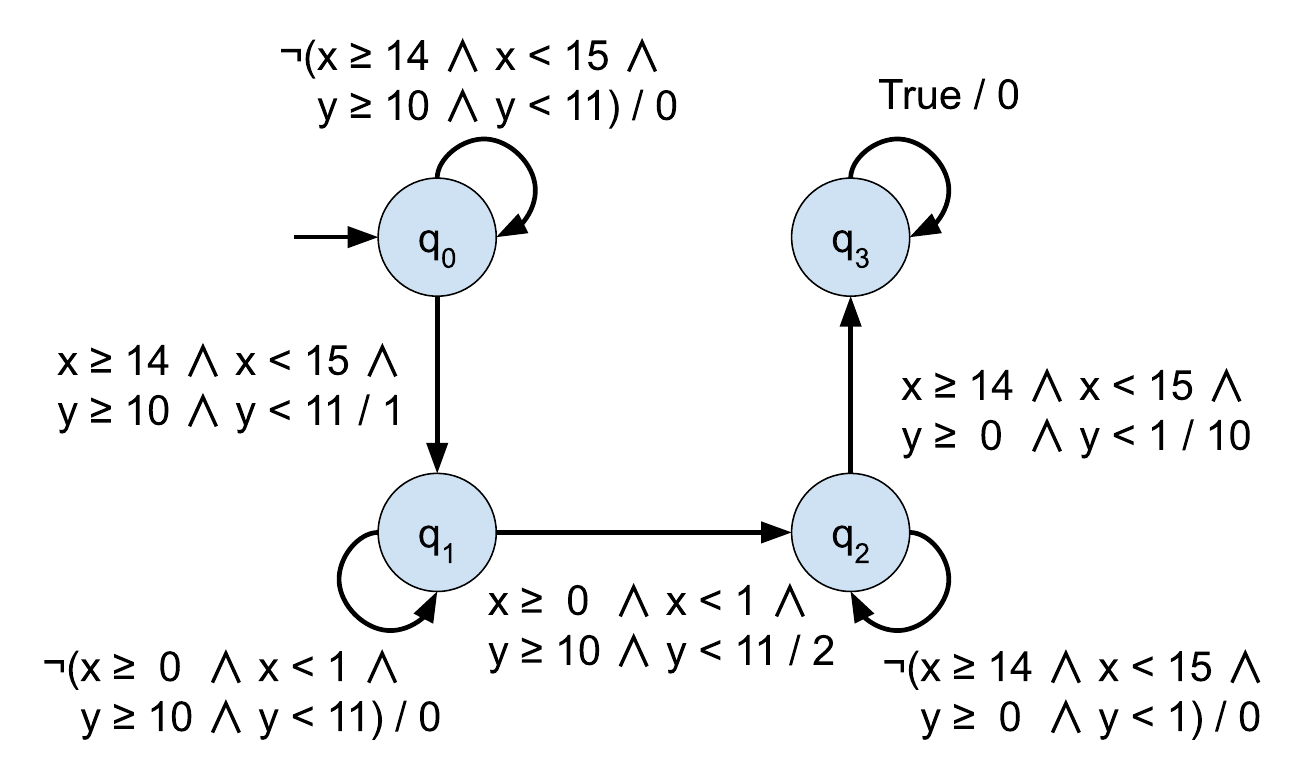}
         \caption{SRM}
    \end{subfigure}
    \caption{RM and SRM of task diagonal\_run for the Office World. Self-loops in the RM with zero rewards are omitted.}
    \label{fig:task_diagonal_run_machines}
\end{figure}
    
We also evaluate our methods that are developed for infinite MDPs in a second environment. We use a modified version of the Mountain Car environment, which is a widely-used environment in the RL community and implemented in the popular Gymnasium library \cite{towers2025gymnasium}. The first version of the Mountain Car environment was defined by Moore \cite{Moore.1990}. We take the implementation of the Mountain Car environment of the Gymnasium library\footnote{\url{https://gymnasium.farama.org/environments/classic_control/mountain_car_continuous/}} and modify it. Figure \ref{fig:mountain_car_environments} illustrates the difference between the original variant and our variant of the environment.
We shift the x-axis to have a left and a right mountain peak, as opposed to a single incline to the right in the original. The selected action is used for four timesteps in our version.
We denote the task for the Mountain Car environment as `rml'. The name `rml' refers to `right-middle-left'. The task is defined as follows. First, the car has to be navigated onto the right mountain. Afterwards, it has to be navigated back to the start region between the mountains. Thereby, it is required that the velocity is back to near zero. Finally, the car has to be navigated onto the left mountain. We use this environment exclusively to evaluate our novel methods for infinite MDPs. Hence, we only define an SRM version of this task that is visualized in Figure \ref{fig:SRM_Task_rml}.

\begin{figure}
    \centering
    \begin{subfigure}[b]{0.45\linewidth}
        \includegraphics[height=3.5cm,frame]{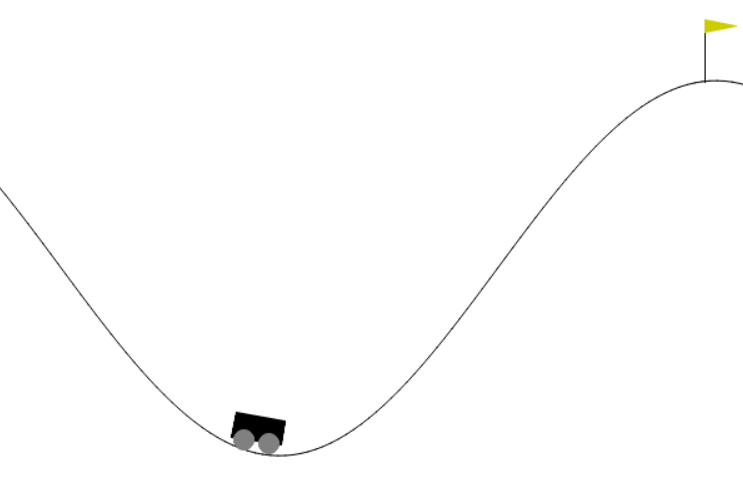}
    \end{subfigure}
    \hspace{0.08\linewidth}
    \begin{subfigure}[b]{0.45\linewidth}
         \includegraphics[height=3.5cm,frame]{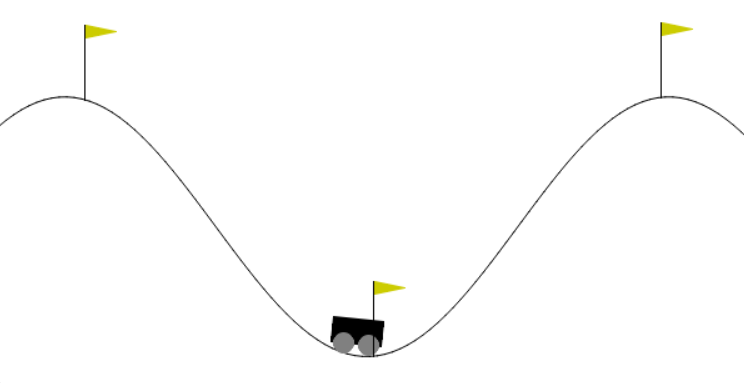}
    \end{subfigure}
    \caption{Original Mountain Car environment (left) and our version (right).}
    \label{fig:mountain_car_environments}
\end{figure}

\begin{figure}
    \centering
    \includegraphics[width=0.6\linewidth]{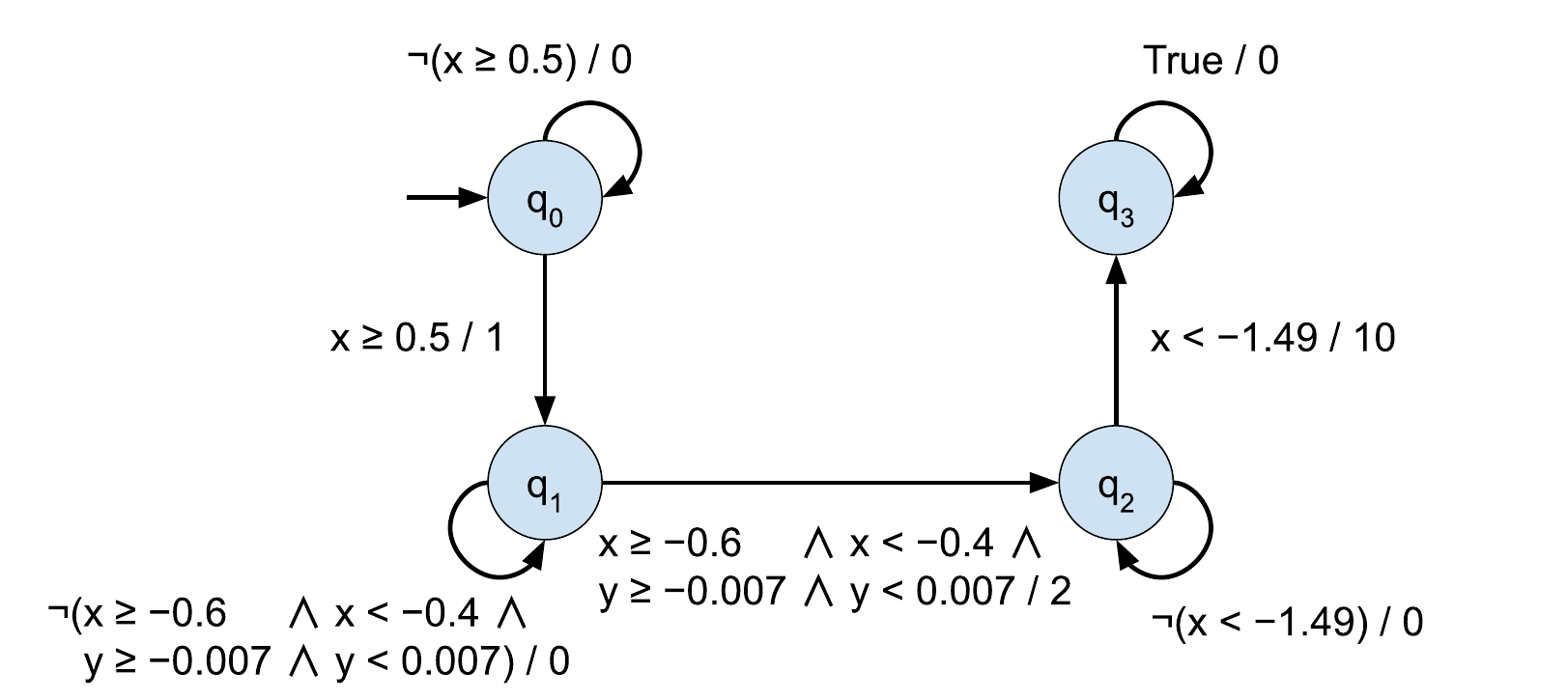}
    \caption{SRM for task `rml' for our Mountain Car environment.}
    \label{fig:SRM_Task_rml}
\end{figure}

\subsubsection{Metric}
To evaluate the methods, we compute a metric that we call `mean10 performance value'. Every 5,000 steps (for Q-Learning, QRM, QSRM, LSRM methods for finite environments) or every 10,000 steps (for DQN \cite{mnih2015humanlevel} (with a frame stack of 100 states), DQRM, DQSRM, LSRM methods for infinite environments), we pause the learning process and execute the current policy as a greedy policy 20 times in a newly initialized environment, where each run is limited to a maximum of 500 steps. The mean reward obtained across these 20 runs is denoted as the `performance value'. The `mean10 performance value' is defined as the mean of the last 10 performance values, normalized to the range $[0,1]$ with respect to the maximal achievable reward of the task. We use this metric for the following reasons:
\begin{enumerate}
    \item It provides an intuitive view of the progress during training since we compute it periodically.
    \item It directly refers to the performance and effectiveness of the returned policy because it is not affected by the exploration parameter $\epsilon$.
    \item Averaging 20 runs and using a sliding window smooths the metric value and detaches the value from outliers or stuck runs.
    \item Scaling between 0 and 1 improves interpretability by bounding the metric and making relative performance comparisons more intuitive.
\end{enumerate}

\subsubsection{Research Questions}
We summarize the research questions for the evaluation in the following:
\begin{enumerate}[label=RQ\arabic*:,ref=RQ\arabic*]
    \item How do our methods perform in comparison to the baseline methods Q-Learning and DQN (with a frame stack of 100 states for DQN)?\label{RQ1}
    \item How does (D)QSRM perform in comparison to (D)QRM? We expect that the methods for SRMs perform exactly the same as the methods for RMs. The reason for that is that QSRM can be seen as a more intuitive QRM method in the case where the agent has access to the RM and SRM. The reason for that is if we substitute each formula with a label and add a labeling function into the environment that exactly maps the values regarding to the formulas, we get the setup of QRM. Accordingly, QRM and QSRM have to generate exactly the same results if we use the same random seeds.\label{RQ2}
    \item LSRM learns policies end-to-end. The obvious question is whether our LSRM methods learn an optimal policy.\label{RQ3}
    \item LSRM learns the SRM in the learning process, too. Here, the main question is: Does LSRM learn an almost surely equivalent SRM? If this is not the case, the follow up question is: Is the learned SRM close enough to the one used to get good performance values?\label{RQ4}
\end{enumerate}

\subsection{Results of methods with access to RM/SRM}
In this section, we focus on the methods that have direct access to the RM or SRM. Thus, we look at (D)QRM, (D)QSRM, and for comparison at the baseline methods Q-Learning and DQN (with a frame stack of 100 states for DQN). These methods relate to \ref{RQ1} and \ref{RQ2}.
Figure \ref{fig:mean10_performance_values_methods_with_access_to_rm_srm} shows the results for the different methods. The figure provides the required information and results to evaluate and compare our methods with the baseline and existing ones.

The baseline methods, Q-Learning and DQN (with a frame stack of 100 states), perform worse. Q-Learning has no information about the state history. Thus, it cannot learn non-Markovian reward functions, which results in low mean10 performance values. Interestingly, DQN with a frame stack also performs worse in the example tasks, even though it has access to the last 100 states, which should be sufficient to cover all required steps for the tasks in the frame stack. If we increase the frame stack further, it becomes increasingly challenging to train the network. In addition, we never know how large the frame stack needs to be. As a result, both baseline methods perform worse on the example tasks. The methods based on RMs and SRMs train effective policies and obtain good mean10 performance values. As a result of this comparison, we can say that our new methods (D)QSRM and the existing methods (D)QRM outperform the baseline methods. Furthermore, QRM and QSRM converge to the optimal mean10 performance value 1. This observation also matches with the convergence properties of QRM \cite{Icarte.2018} and QSRM (Theorem \ref{theorem_QSRM_convergence}).
The second observation for QRM and QSRM is that they perform exactly the same. This is a result that we expected in \ref{RQ2}. Furthermore, it demonstrates that our new QSRM method can also train effective policies in the non-Markovian setting.

The result regarding performance also applies to the DQRM and DQSRM methods for infinite state spaces. As expected, both methods yield identical mean10 performance values when initialized with the same random seed. We note that neither DQRM nor DQSRM reach optimal mean10 performance values. Thus, their policies are not optimal. However, we expected this behaviour because both methods rely on DQN. Nevertheless, they still achieve good performance values.

Overall, it can be seen that our new methods (D)QSRM provide the ability to train effective policies and obtain good performance values and at the same time have fewer restrictions than (D)QRM as they respect the standard MDP definition and do not need a labeling function. In addition, like (D)QRM, they outperform the baseline methods.

\begin{figure}
    \centering
    \begin{subfigure}[b]{\linewidth}
        \includegraphics[width=\linewidth]{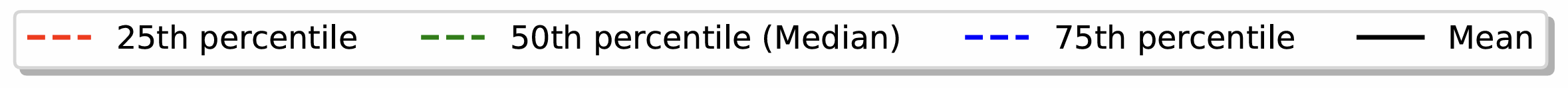}
    \end{subfigure}\\
    \begin{subfigure}[b]{0.32\linewidth}
    \caption*{Q-Learning}
        \includegraphics[width=\linewidth]{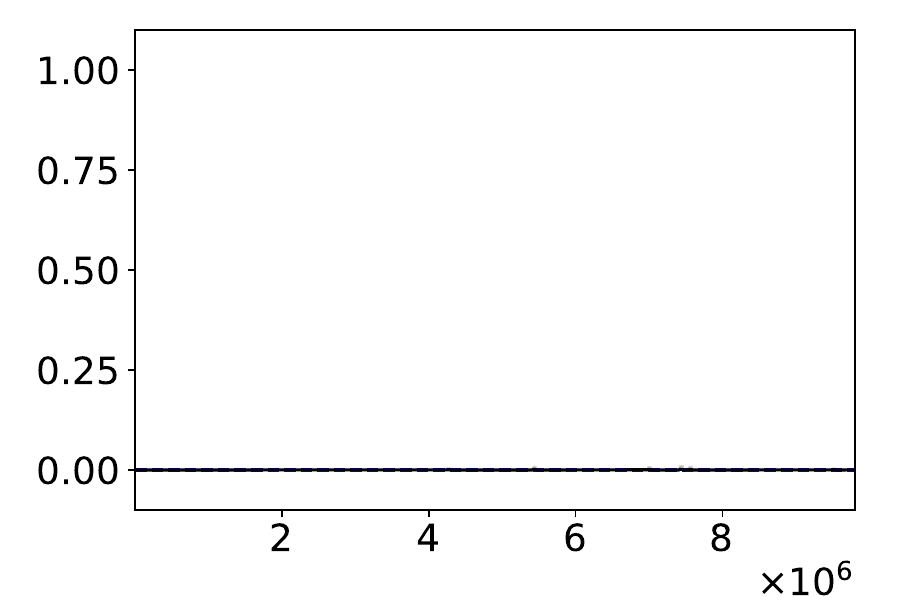}
    \end{subfigure}
    \hfill
    \begin{subfigure}[b]{0.32\linewidth}
        \caption*{QRM}
        \includegraphics[width=\linewidth]{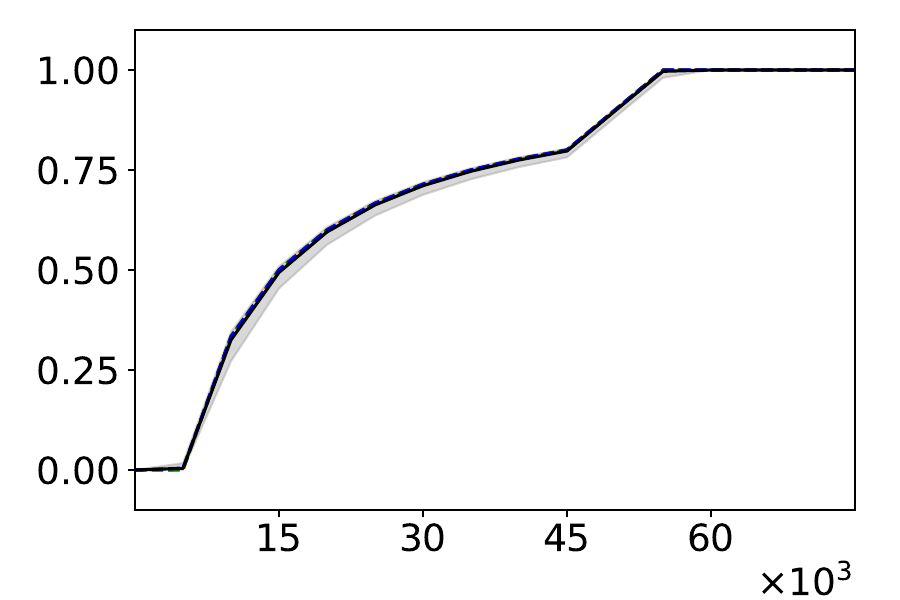}
    \end{subfigure}
    \hfill
    \begin{subfigure}[b]{0.32\linewidth}
        \caption*{QSRM}
        \includegraphics[width=\linewidth]{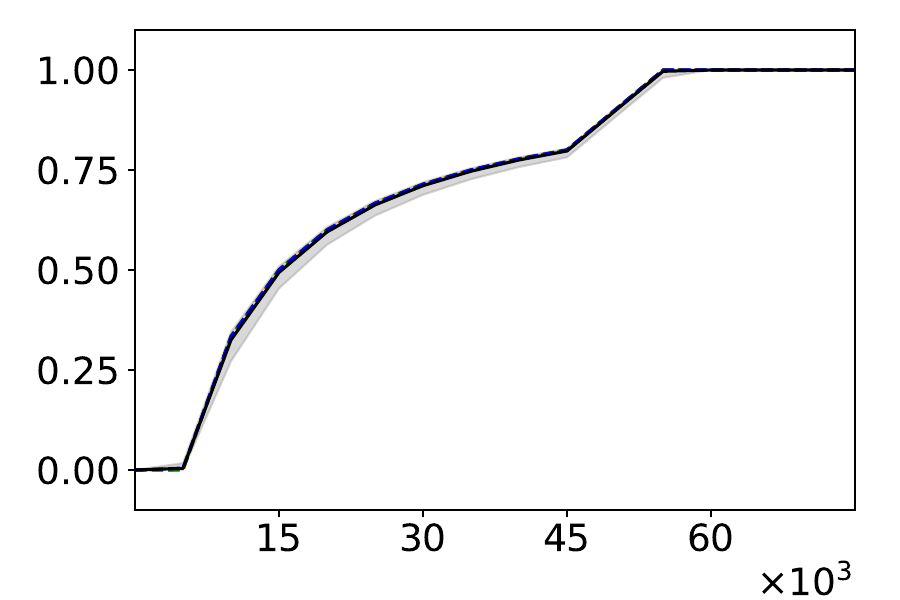}
    \end{subfigure}\\
    \begin{subfigure}[b]{0.32\linewidth}
        \includegraphics[width=\linewidth]{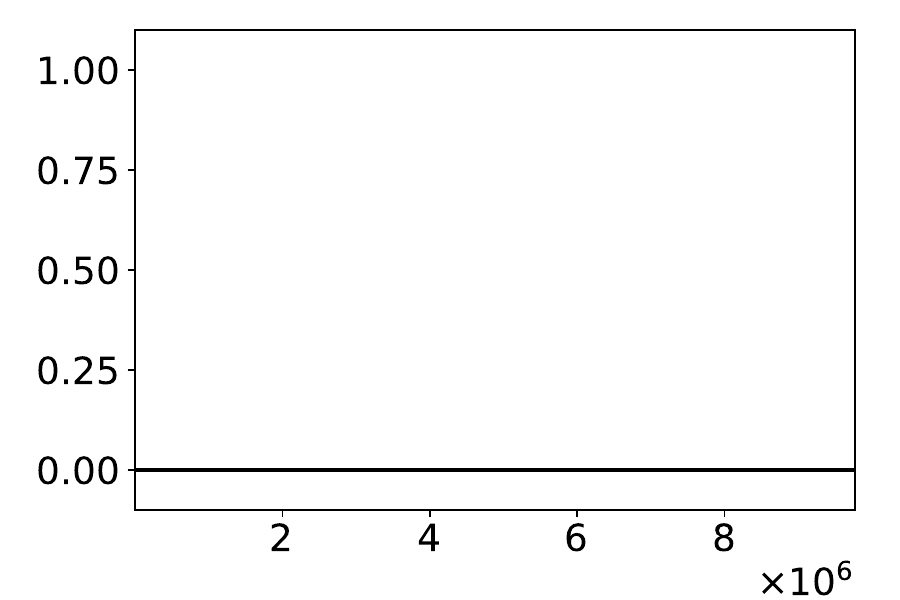}
    \end{subfigure}
    \hfill
    \begin{subfigure}[b]{0.32\linewidth}
        \includegraphics[width=\linewidth]{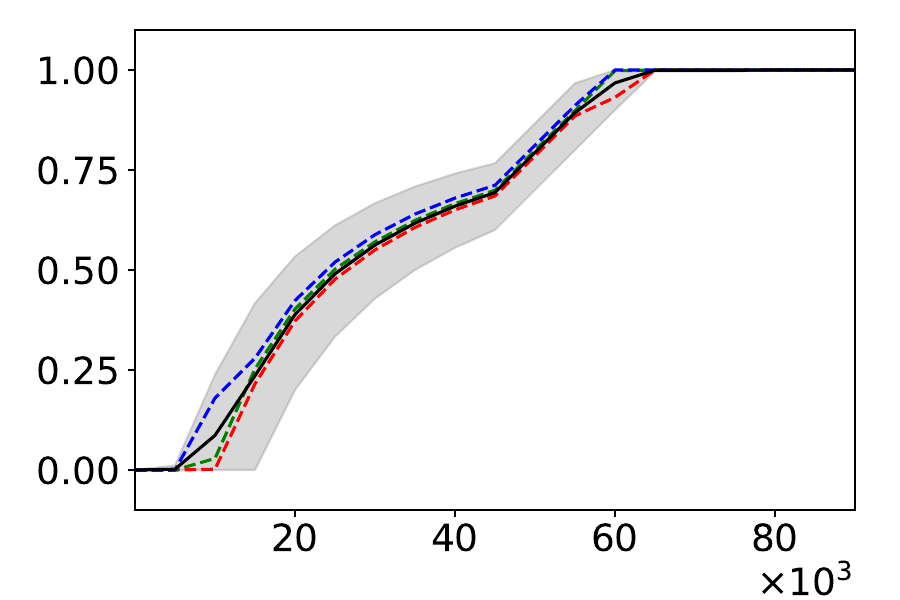}
    \end{subfigure}
    \hfill
    \begin{subfigure}[b]{0.32\linewidth}
        \includegraphics[width=\linewidth]{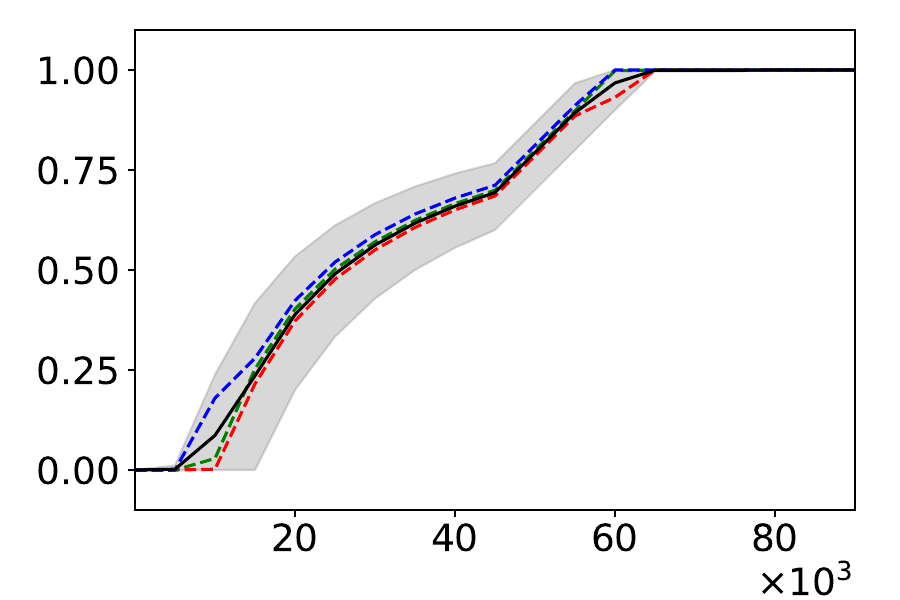}
    \end{subfigure}\\[0.5em]
    \hrule
    \begin{subfigure}[b]{0.32\linewidth}
        \caption*{DQN}
        \includegraphics[width=\linewidth]{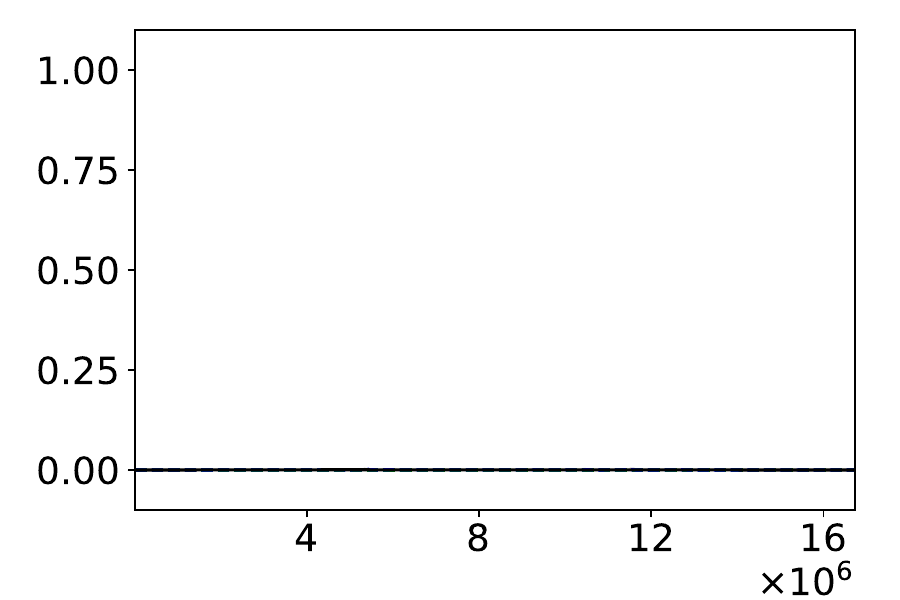}
    \end{subfigure}
    \hfill
    \begin{subfigure}[b]{0.32\linewidth}
        \caption*{DQRM}
        \includegraphics[width=\linewidth]{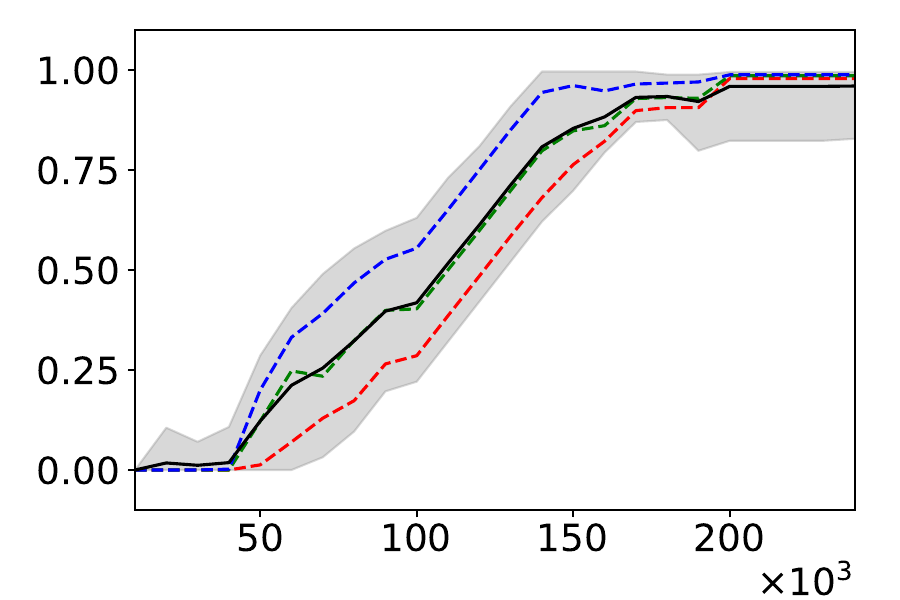}
    \end{subfigure}
    \hfill
    \begin{subfigure}[b]{0.32\linewidth}
        \caption*{DQSRM}
        \includegraphics[width=\linewidth]{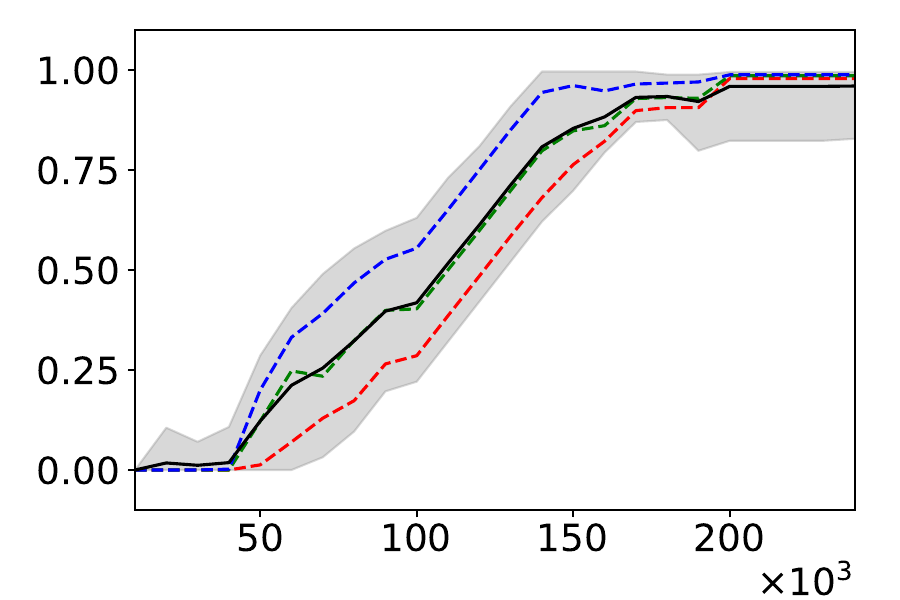}
    \end{subfigure}\\
    \begin{subfigure}[b]{0.32\linewidth}
        \includegraphics[width=\linewidth]{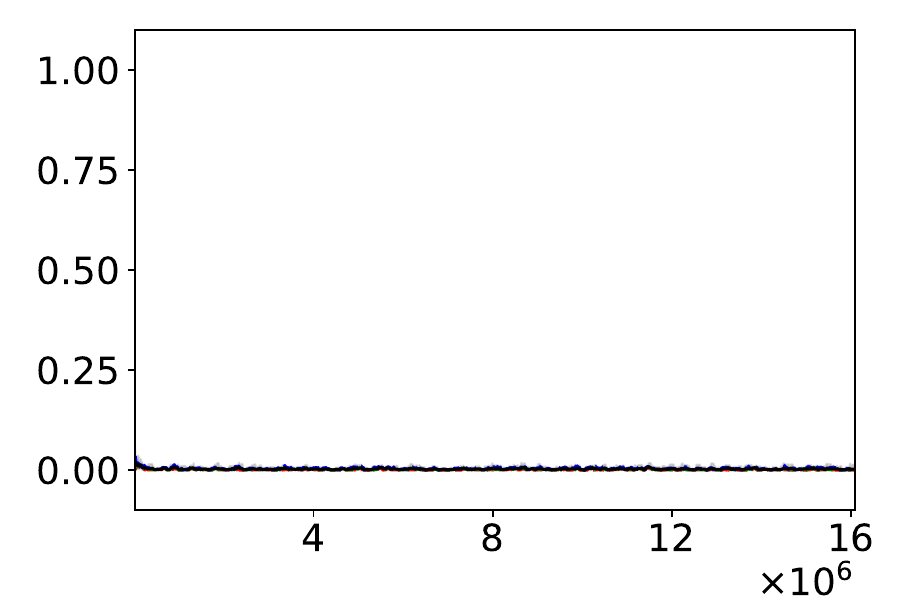}
    \end{subfigure}
    \hfill
    \begin{subfigure}[b]{0.32\linewidth}
        \includegraphics[width=\linewidth]{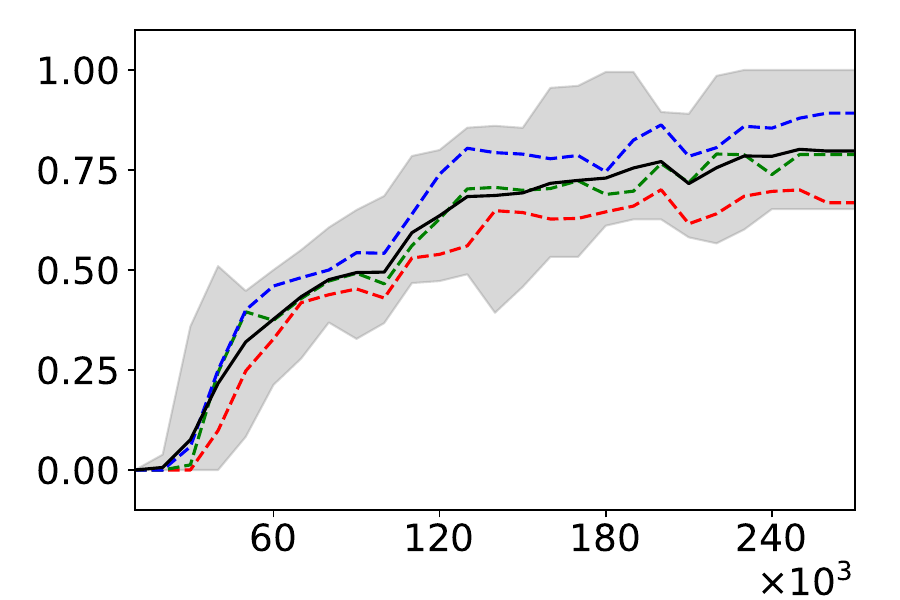}
    \end{subfigure}
    \hfill
    \begin{subfigure}[b]{0.32\linewidth}
        \includegraphics[width=\linewidth]{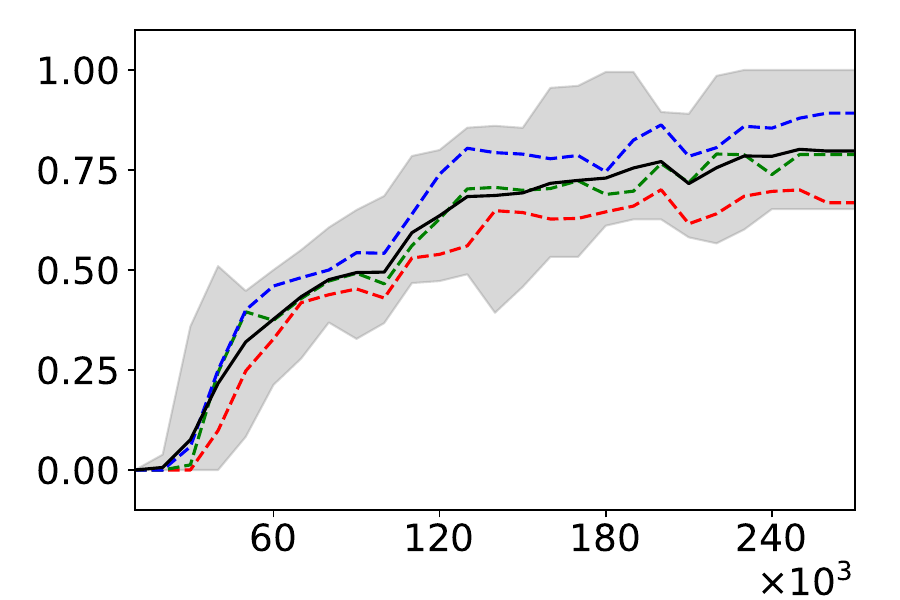}
    \end{subfigure}\\
   \begin{subfigure}[b]{0.32\linewidth}
        \includegraphics[width=\linewidth]{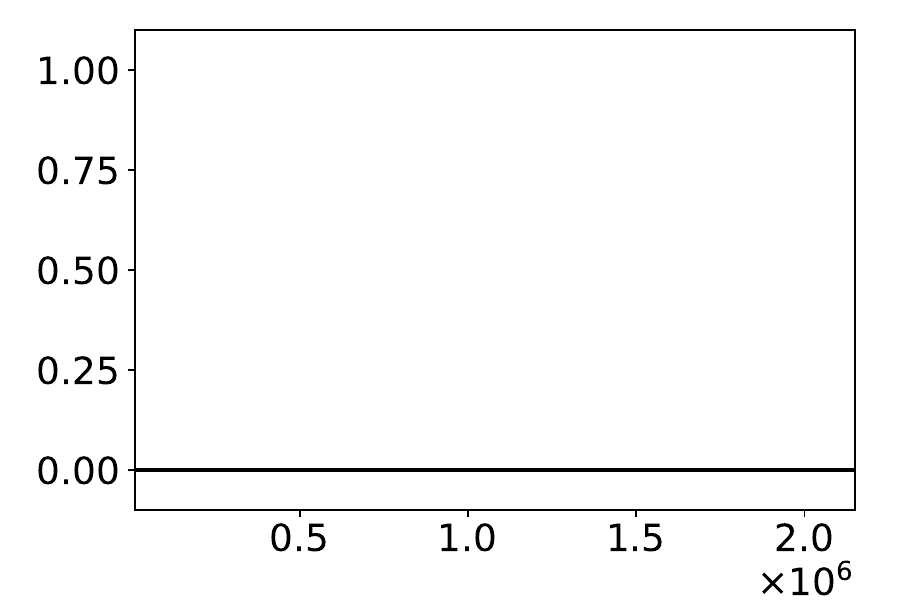}
    \end{subfigure}
    \hfill
    \begin{subfigure}[b]{0.32\linewidth}
        \caption*{}
    \end{subfigure}
    \hfill
    \begin{subfigure}[b]{0.32\linewidth}
        \includegraphics[width=\linewidth]{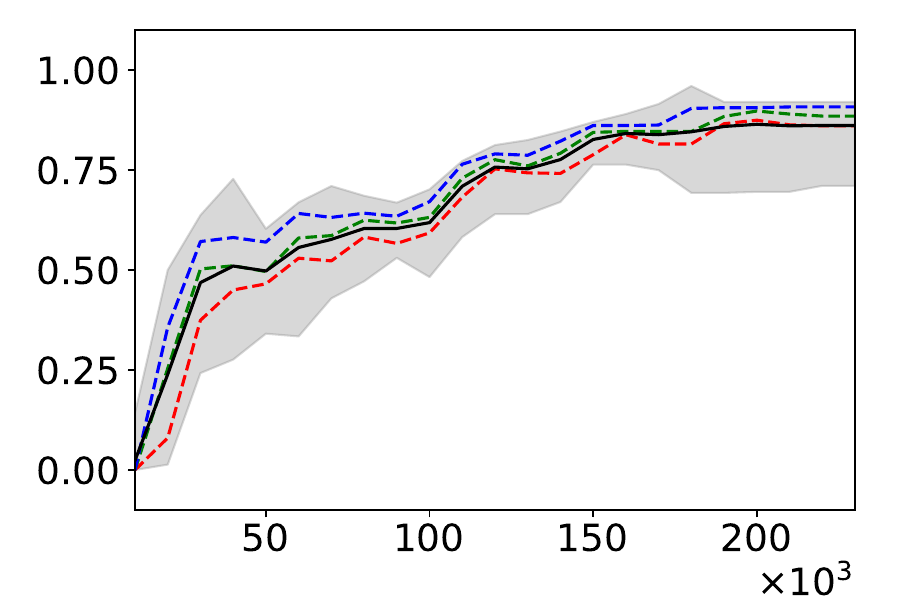}
    \end{subfigure}

    \caption{Mean10 performance values for the tasks post\_inner\_offices and diagonal\_run in the discrete Office World (Rows 1-2), and for the tasks post\_inner\_offices, diagonal\_run, and rml in the infinite environments (Rows 3-5).}
    \label{fig:mean10_performance_values_methods_with_access_to_rm_srm}
\end{figure}

\subsection{Results of LSRM}
In this section, we evaluate our novel LSRM methods that can infer the SRM during the learning process, and thus, provide the ability to learn policies for MDPs with non-Markovian reward functions end-to-end. In particular, we consider \ref{RQ3} and \ref{RQ4}: whether the LSRM methods learn an optimal policy, and whether the inferred SRM is almost surely equivalent to the one used in the environment. If they are not almost surely equivalent, we check whether the learned SRM is close enough to the one used to obtain good performance values.

Figure \ref{fig:mean10_performance_values_LSRM_methods} shows the mean10 performance values for all of our LSRM methods. The result plots demonstrate the performance and effectiveness of them. We can observe similar results as for the methods with access to the SRM. LSRM converges in both variants, LSRM-GF and LSRM-FT, for the tasks in the finite environment to the optimal mean10 performance values. Furthermore, LSRM reaches good but not optimal performance values in the infinite environments. Regarding \ref{RQ3}, we can claim that LSRM learns an optimal policy in the finite setting but not in the infinite setting. Nevertheless, LSRM learns a very good policy in the infinite setting. These results also match the convergence properties of LSRM from Section \ref{LSRM_Convergence}.

\begin{figure}
    \centering
    \begin{subfigure}[b]{\linewidth}
        \includegraphics[width=\linewidth]{graphics/results/legend.png}
    \end{subfigure}\\
    \begin{subfigure}[b]{0.32\linewidth}
    \caption*{post\_inner\_offices}
        \includegraphics[width=\linewidth]{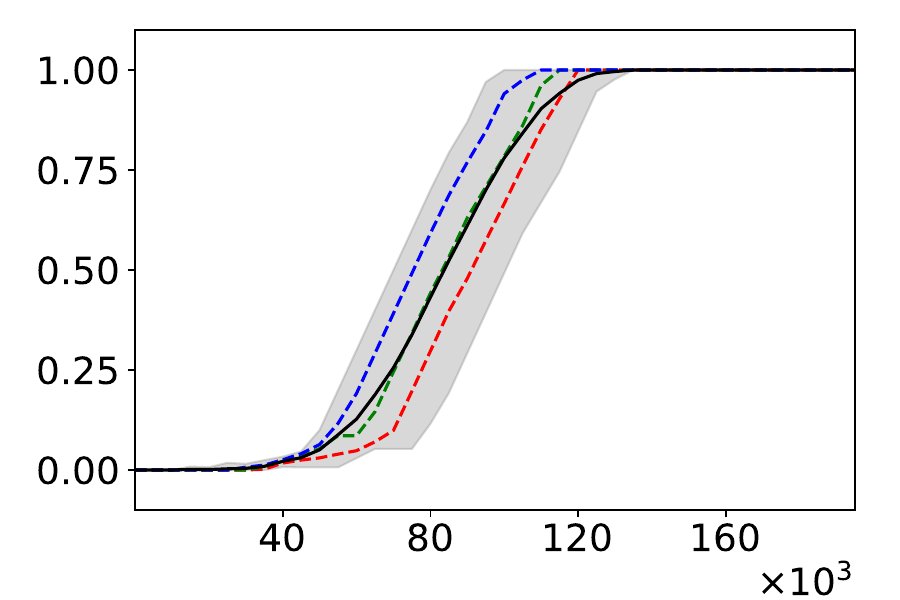}
    \end{subfigure}
    \hfill
    \begin{subfigure}[b]{0.32\linewidth}
        \caption*{diagonal\_run}
        \includegraphics[width=\linewidth]{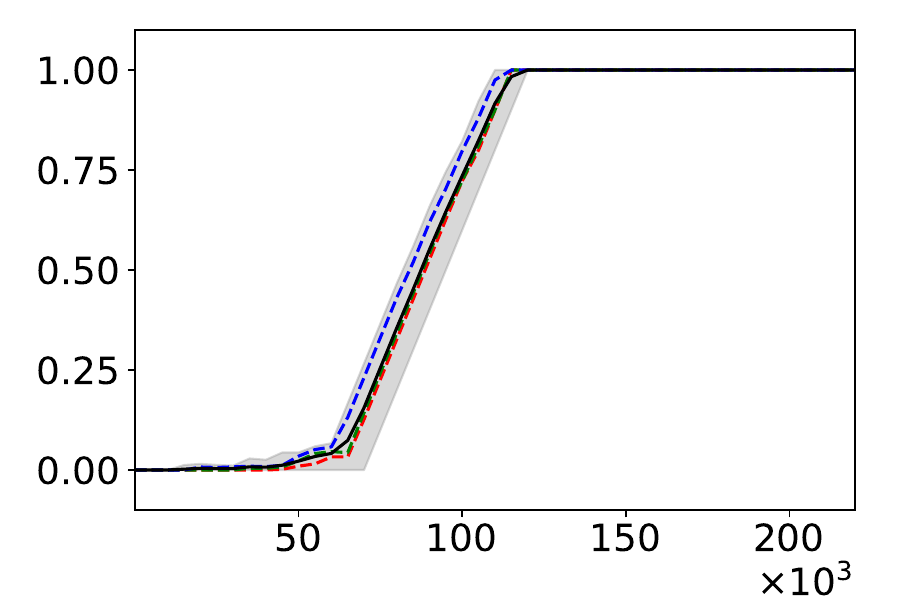}
    \end{subfigure}
    \hfill
    \begin{subfigure}[b]{0.32\linewidth}
        \caption*{}
    \end{subfigure}\\
    \begin{subfigure}[b]{0.32\linewidth}
        \includegraphics[width=\linewidth]{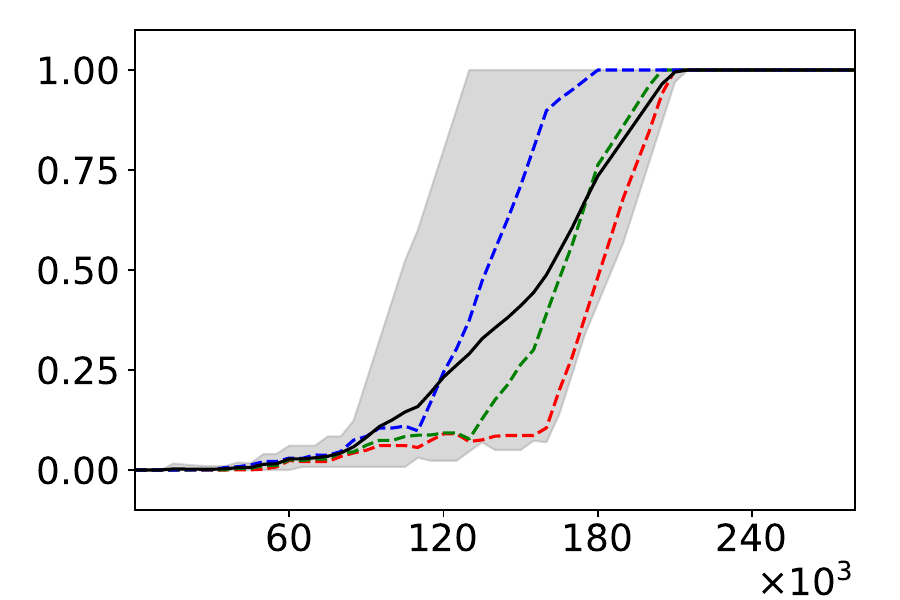}
    \end{subfigure}
    \hfill
    \begin{subfigure}[b]{0.32\linewidth}
        \includegraphics[width=\linewidth]{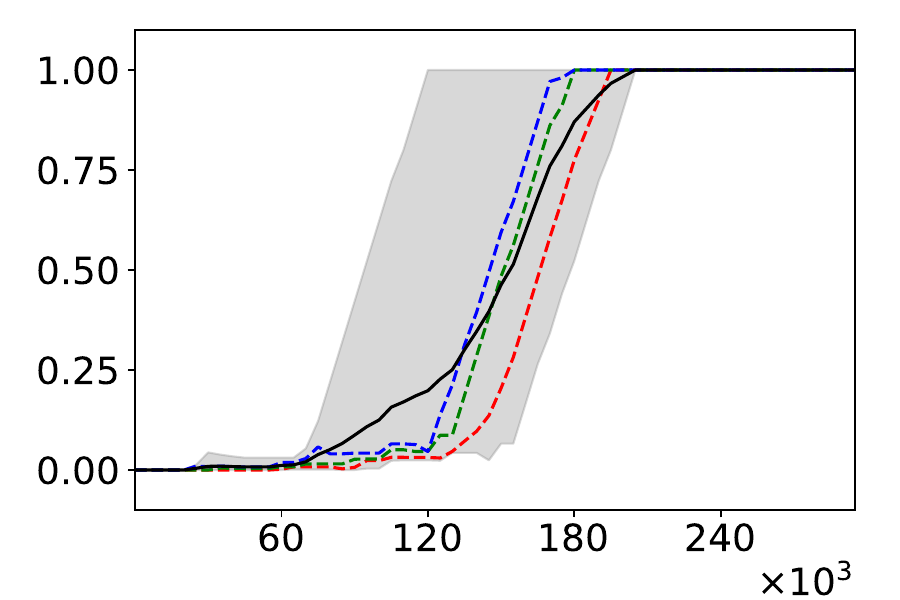}
    \end{subfigure}
     \hfill
    \begin{subfigure}[b]{0.32\linewidth}
        \caption*{}
    \end{subfigure}\\[0.5em]
    \hrule
    \begin{subfigure}[b]{0.32\linewidth}
        \caption*{post\_inner\_offices}
        \includegraphics[width=\linewidth]{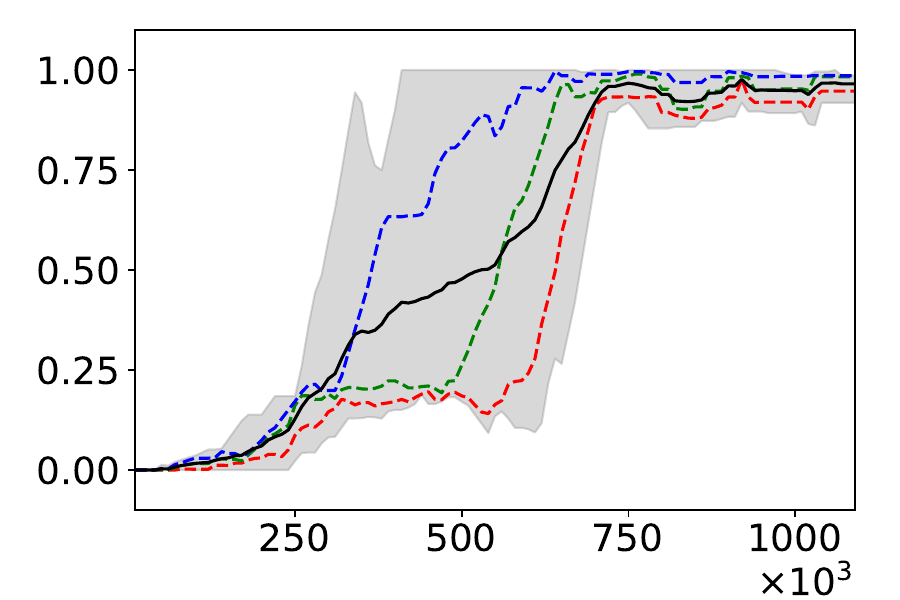}
    \end{subfigure}
    \hfill
    \begin{subfigure}[b]{0.32\linewidth}
        \caption*{diagonal\_run}
        \includegraphics[width=\linewidth]{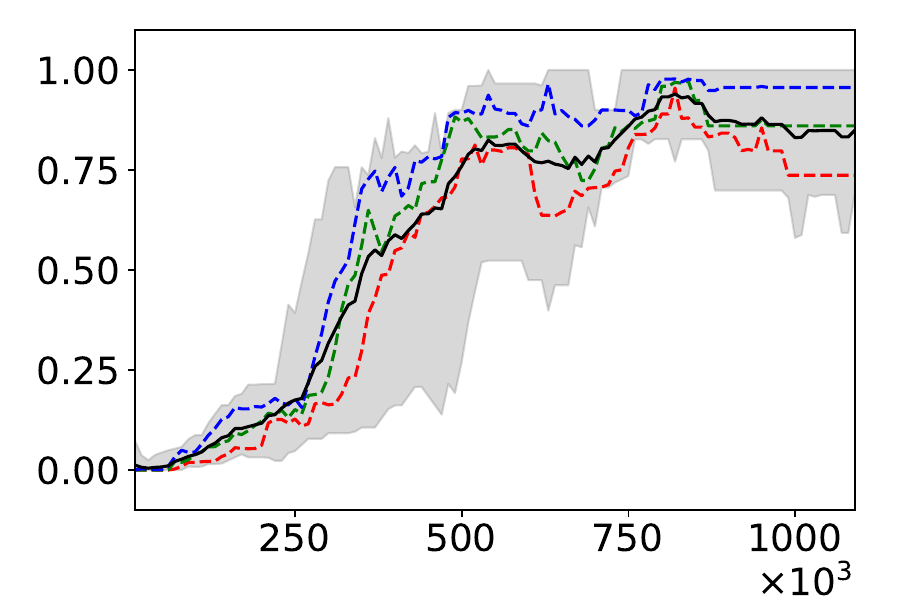}
    \end{subfigure}
    \hfill
    \begin{subfigure}[b]{0.32\linewidth}
        \caption*{rml}
        \includegraphics[width=\linewidth]{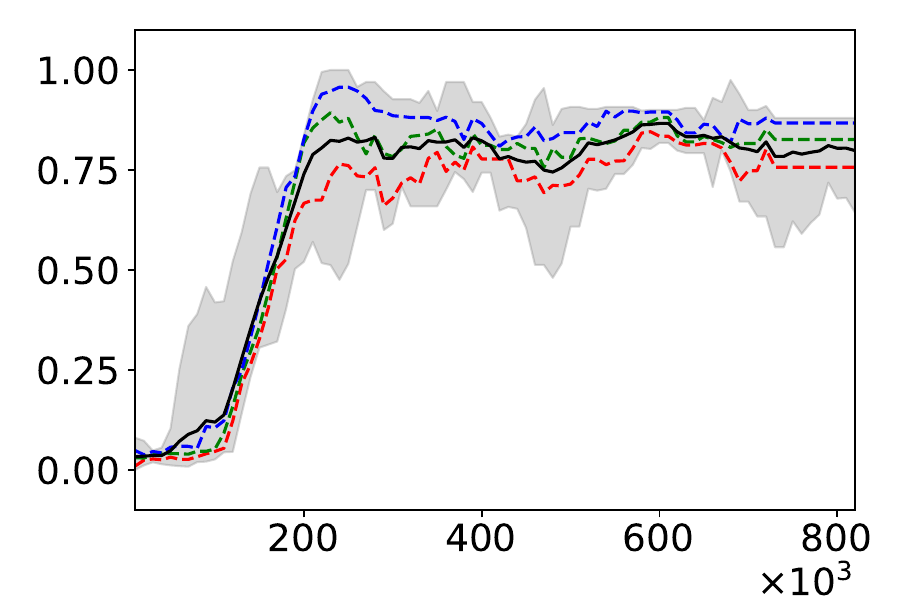}
    \end{subfigure}\\
    \begin{subfigure}[b]{0.32\linewidth}
        \includegraphics[width=\linewidth]{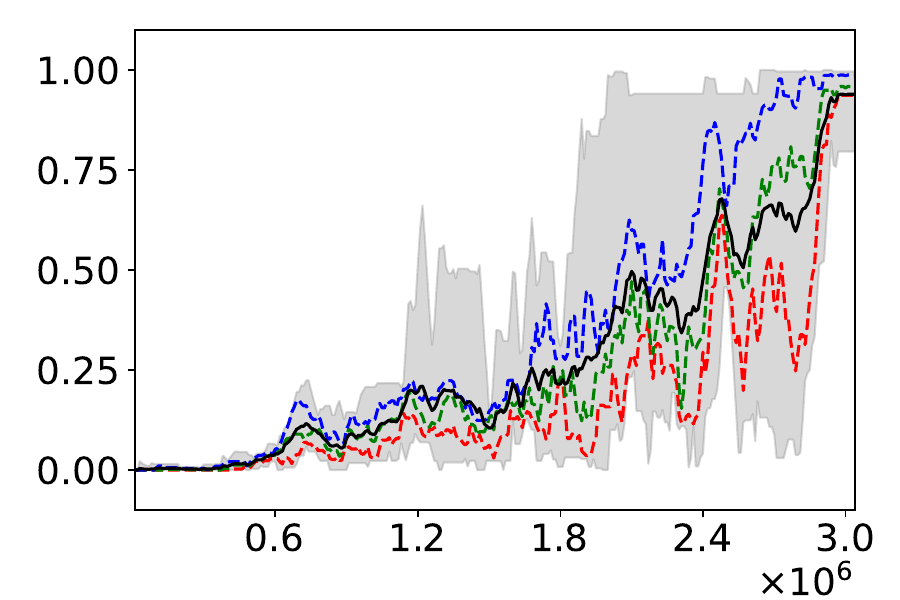}
    \end{subfigure}
    \hfill
    \begin{subfigure}[b]{0.32\linewidth}
        \includegraphics[width=\linewidth]{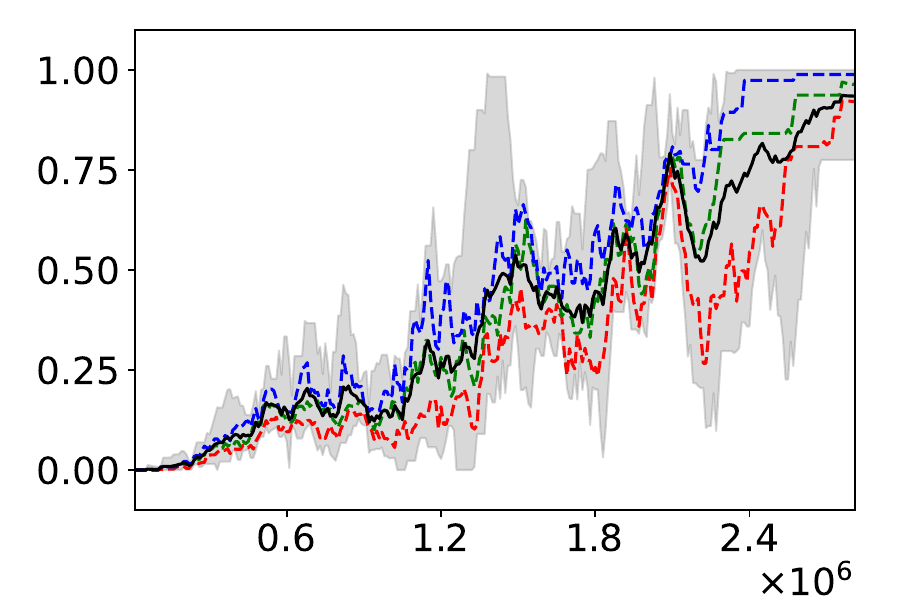}
    \end{subfigure}
    \hfill
    \begin{subfigure}[b]{0.32\linewidth}
        \includegraphics[width=\linewidth]{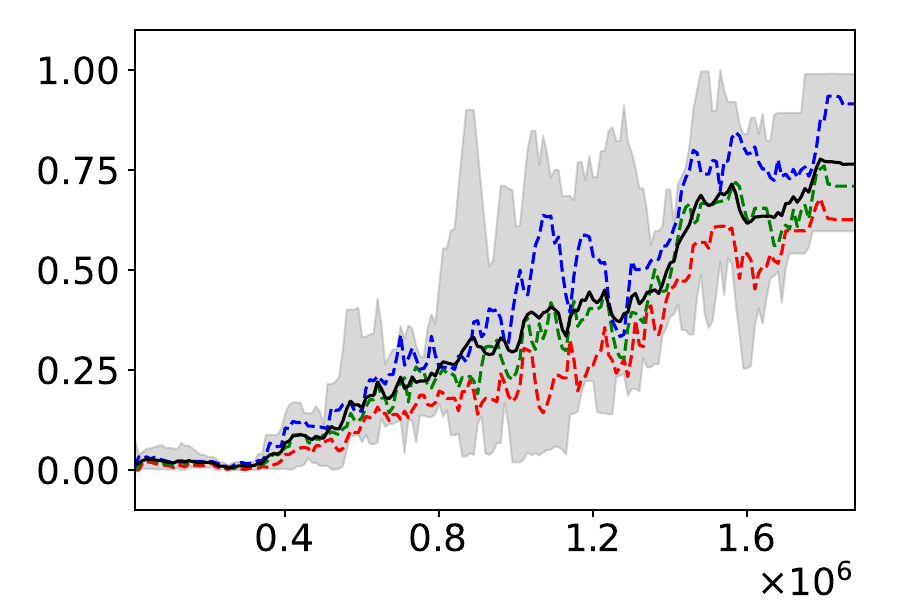}
    \end{subfigure}
    
    \caption{LSRM mean10 performance values: Finite state space environment in the upper half and infinite state space environments in the lower half. 
    The results for LSRM-GF are in the first row and for LSRM-FT in the second row.}
    \label{fig:mean10_performance_values_LSRM_methods}
\end{figure}

\ref{RQ4}, which addresses whether LSRM learns an almost surely equivalent SRM, is still not answered. So far, we observe that the agent trained by LSRM reaches optimal performance values in the finite setting and good performance values in the infinite setting. LSRM must have learned an SRM that is close enough to the one used because it would not reach these performance values if it had learned an SRM that was far from the true one. However, the natural question is how do the inferred SRMs look in comparison to the SRMs that are used in the environment? This is a crucial question because LSRM provides the SRM to the user, and thus, interesting information about the reward structure. This gives the user the ability to make decisions or conclusions about the environment.
The figures \ref{fig:LSRM_inferred_SRM_Formula_Templates_Finite}, \ref{fig:LSRM_inferred_SRM_Formula_Templates_Infinite}, and \ref{fig:LSRM_inferred_SRMs_Given_Formulas} show the SRMs that are learned by LSRM for the task post\_inner\_offices.
Note that only the SRMs for the task post\_inner\_offices are compared, since the observations for the other tasks are the same.
The SRM that is used in the environment for the experiment, and thus, is the SRM to compare the learned SRMs by LSRM with, is the SRM in Figure \ref{fig:post_inner_offices_env_srm}. For simplicity, we name it SRM-Real in the following.
First, we compare the SRMs that are generated by LSRM-GF. Afterwards, we discuss and analyze the SRMs that are generated by LSRM-FT.
\subsubsection{SRMs by LSRM-GF}
First, we compare the SRMs that LSRM-GF learned. They are visualized in Figure \ref{fig:LSRM_inferred_SRMs_Given_Formulas}. For simplicity, we name the learned SRM in the finite environment SRM-GF-Fin and in the infinite environment SRM-GF-Inf. We can observe that
\begin{enumerate}
    \item \label{SRMs_by_LSRM-GF_1} SRM-GF-Fin is not equal to SRM-Real. They differ in the aspect that SRM-GF-Fin has one state less and the transition with the output 10 is a self-loop in state 2 instead of a new state with a self-loop that always outputs 0 as in SRM-Real. Nevertheless, they are almost surely equivalent. All episodes end if the goal is reached and a reward of 10 is observed. This means that all episodes end after the SRM uses the transition with the output value of 10. Consequently, SRM-GF-Fin and SRM-Real are equivalent on all attainable trajectories, and thus, they are almost surely equivalent.
    \item The results and arguments of the previous point~\ref{SRMs_by_LSRM-GF_1} also apply in the comparison of SRM-GF-Inf and SRM-Real. The difference is that the names of states 1 and 2 are swapped, and the transition that outputs a reward of 10 goes back to state 2. Nevertheless, this difference does not change the result that SRM-GF-Inf and SRM-Real are not equal but they are almost surely equivalent.
\end{enumerate}

\subsubsection{SRMs by LSRM-FT}
Here, we compare the SRMs that LSRM-FT learned. The SRMs are visualized in Figure \ref{fig:LSRM_inferred_SRM_Formula_Templates_Finite} (finite environment) and Figure \ref{fig:LSRM_inferred_SRM_Formula_Templates_Infinite} (infinite environment). For simplicity, we name the learned SRM in the finite environment SRM-FT-Fin and in the infinite environment SRM-FT-Inf. We can observe that SRM-FT-Fin and SRM-FT-Inf are both not equal to SRM-Real. The remaining question is whether they are almost surely equivalent. Here, we distinguish between the comparison to SRM-FT-Fin and SRM-FT-Inf and can observe that
\begin{enumerate}
    \item the formulas that are used as guards differ between SRM-Real and SRM-FT-Fin. However, the only reason for this behaviour is that LSRM-FT uses real values for the values in the formula templates and LSRM-FT can only observe discrete positions in the finite environment. So, the formulas are equivalent if we only consider discrete positions. The additional transitions introduced by the post-processing step in Section \ref{LSRM_GF_Generate_SRM_from_Model} only cover non-discrete positions. In a discrete setting, they remain unused and exist only because the completeness check is performed over the full space $\mathbb{R}$. As a result, SRM-Real and SRM-FT-Fin are almost surely equivalent according to the definition in Section \ref{LSRM_Convergence}.
    \item SRM-Real and SRM-FT-Inf are not almost surely equivalent. Let us consider the transition from state 0 to state 1 and the self-loop from state 0. SRM-Real and SRM-FT-Inf output a different reward for $X = \{(x,y)| (x \in (4.99994980, 5) \vee x \in ( 6, 6.00046406)) \wedge (y \in (5, 5.00200890) \vee y \in (6, 6.00000048))\}$. The probability of a trajectory $\tau$ that includes a observation of $X$ in the state-sequence and the SRM is in SRM state $0$ is greater than 0. Consequently, $\tau$ is a counterexample to the claim that SRM-Real and SRM-FT-Inf are almost surely equivalent, as the rewards generated by SRM-Real and SRM-FT-Inf differ. However, this observation does not break the convergence properties described in Section \ref{LSRM_Convergence} because we only run LSRM-FT for a finite number of epochs, while our convergence result holds in the limit. However, SRM-FT-Inf is similar to SRM-Real if we compare the structure, transitions and guards. As a result, LSRM can train effective policies without an almost surely equivalent SRM.
\end{enumerate}

\begin{figure}
    \centering
    \begin{subfigure}[b]{\linewidth}
        \includegraphics[width=\linewidth]{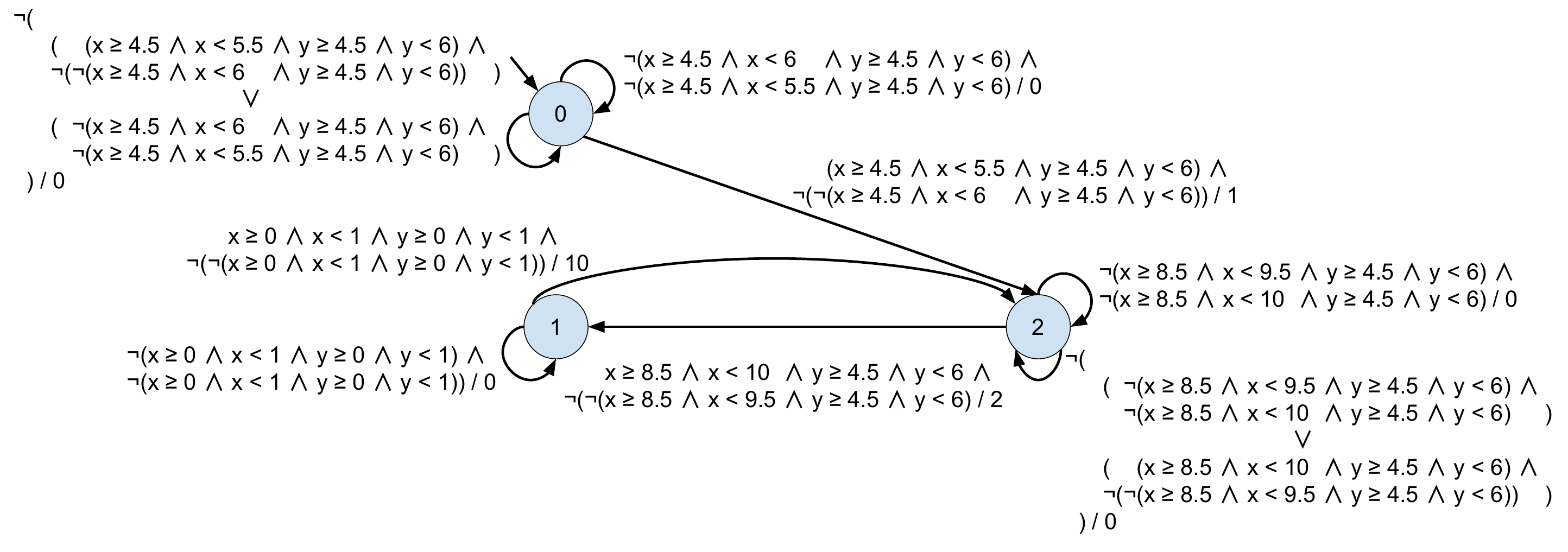}
    \end{subfigure}

    \caption{Learned SRM by LSRM-FT for task post\_inner\_offices in the discrete Office World.}
    \label{fig:LSRM_inferred_SRM_Formula_Templates_Finite}
\end{figure}

\begin{figure}
    \centering
    \begin{subfigure}[b]{\linewidth}
        \includegraphics[width=\linewidth]{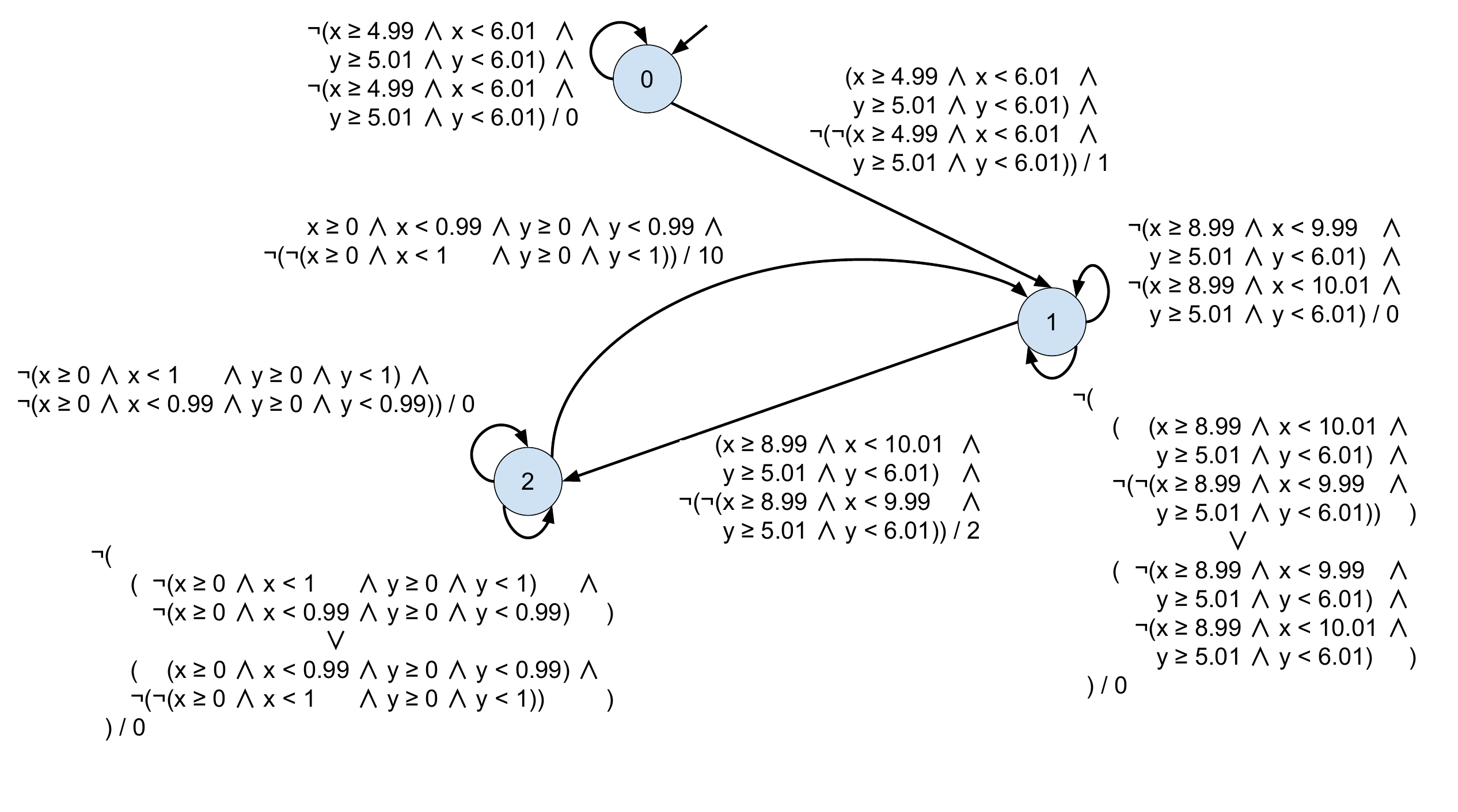}
    \end{subfigure}

    \caption{Learned SRM by LSRM-FT for task post\_inner\_offices in the continuous Office World. For better visualization, we apply our rounding function $\mathrm{round'}$ that rounds the input value to two decimal places. The function $\mathrm{round'}$ is defined by $\mathrm{round'}(x)=
    \begin{cases}
        \frac{\left\lfloor x \cdot 100 \right\rfloor}{100} & \text{, if } \left\lfloor x \cdot 100 \right\rfloor \bmod 10 \geq 5, \\
        \frac{\left\lceil x \cdot 100 \right\rceil}{100} & \text{, if } \left\lfloor x \cdot 100 \right\rfloor \bmod 10 < 5.
    \end{cases}$\\
    The learned SRM is even closer to the SRM used in the environment than the one visualized in this figure.}
    \label{fig:LSRM_inferred_SRM_Formula_Templates_Infinite}
\end{figure}

\begin{figure}
    \centering
    \begin{subfigure}[b]{0.45\linewidth}
        \includegraphics[width=\linewidth]{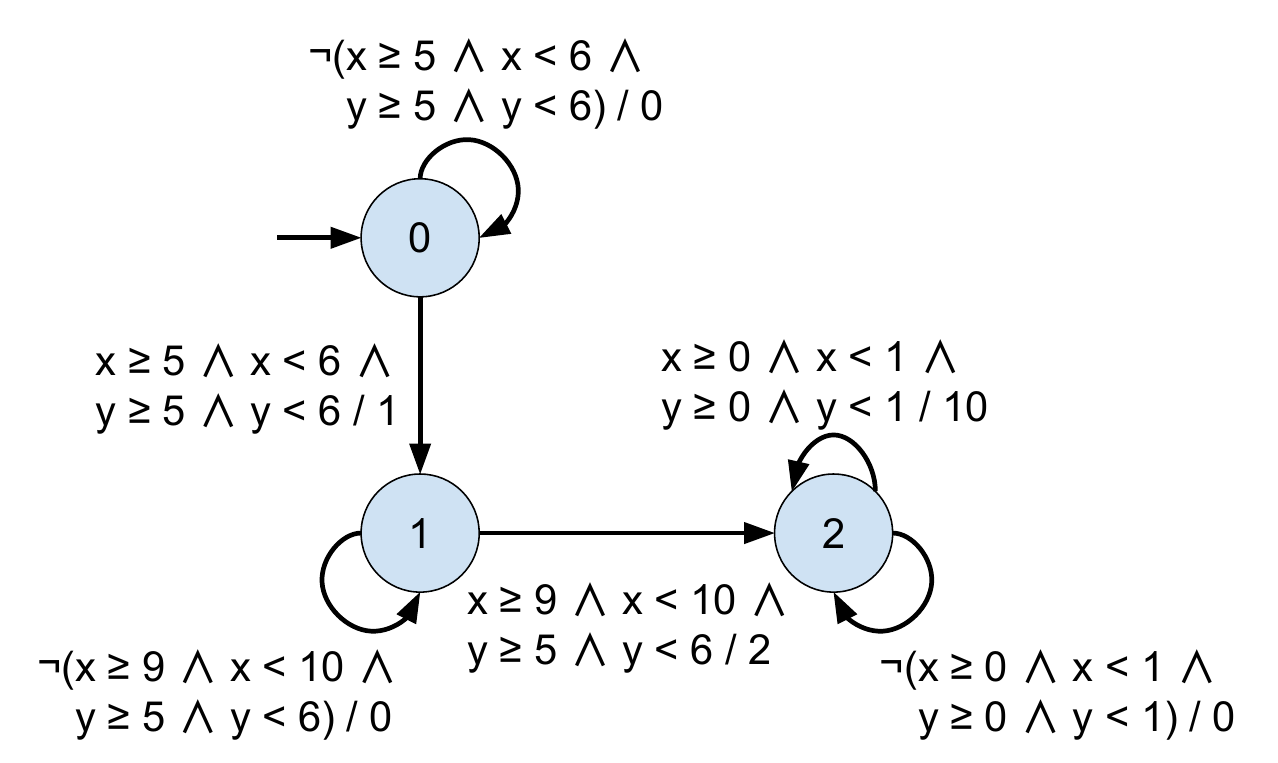}
        \subcaption{Discrete Office World}
    \end{subfigure}
    \hspace{0.05\linewidth}
    \begin{subfigure}[b]{0.45\linewidth}
        \includegraphics[width=\linewidth]{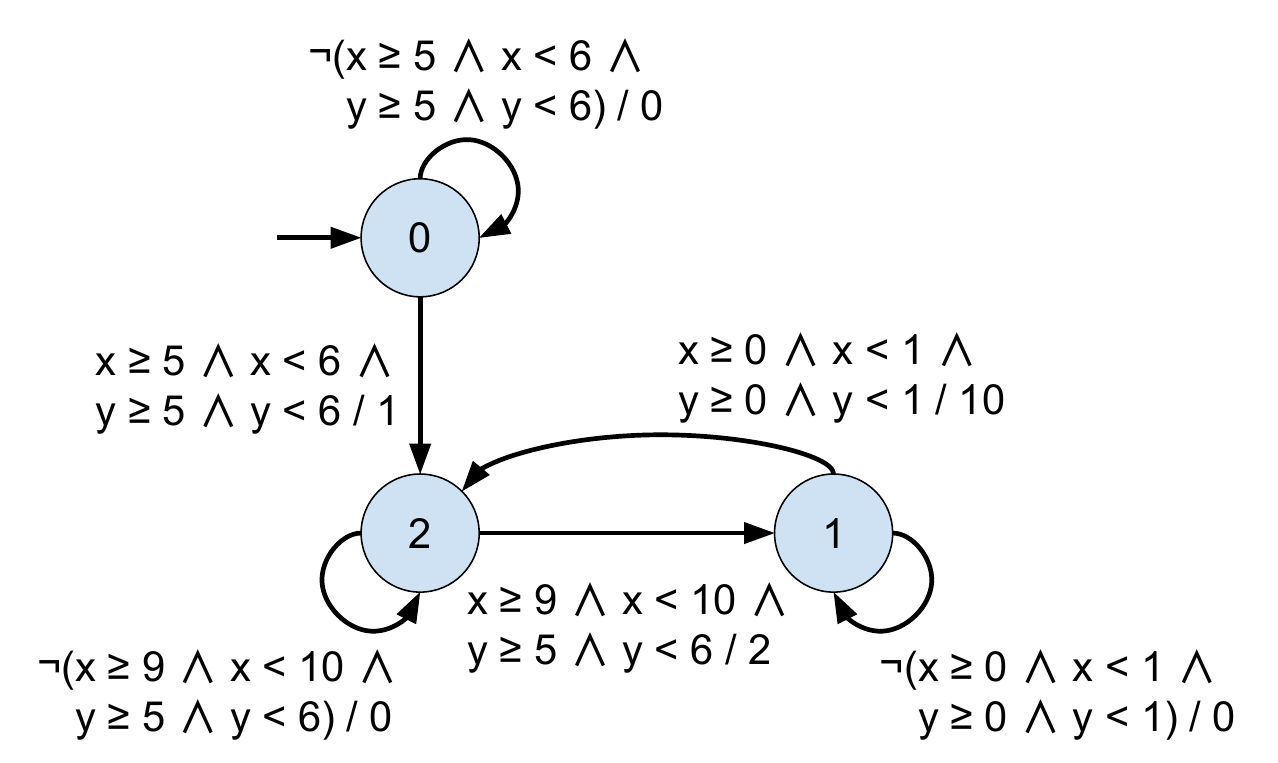}
        \subcaption{Continuous Office World}
    \end{subfigure}

    \caption{Learned SRMs by LSRM-GF for task post\_inner\_offices in the discrete and continuous Office World.}
    \label{fig:LSRM_inferred_SRMs_Given_Formulas}
\end{figure}

Overall, we observe two key points in the evaluation of both LSRM versions. First, our LSRM versions can infer in almost all cases almost surely equivalent SRMs in comparison to the SRMs used in the environments. And second, both LSRM versions can learn effective policies that yield good mean10 performance values.
\newpage
\section{Conclusion}
In this paper, we introduced SRMs as a novel method to represent non-Markovian reward functions. They are more flexible than RMs and only require the standard MDP definition in RL. Furthermore, the formulas that are used as guards in the SRM are more interpretable for the user and do not require generic labels in the environment as opposed to RMs. Based on this novel concept we developed (D)QSRM and LSRM, two powerful methods which leverage the SRM structure to learn policies efficiently. In the experimental evaluation, we demonstrate their applicability, capability and effectiveness to yield good mean10 performance values. LSRM-FT provides the possibility to infer semantic information (symbolic formulas/guards) about the reward function automatically. Accordingly, the user receives important and useful information about the hidden reward function.

\section*{Acknowledgements}
This work has been financially supported by the Research Center Trustworthy Data Science and Security (\mbox{\url{https://rc-trust.ai}}), one of the Research Alliance centers within the UA Ruhr (\mbox{\url{https://uaruhr.de}}). We also thank Jan Corazza and Simon Lutz for helpful discussions and suggestions.

%
%
%

\bibliographystyle{splncs04}
\bibliography{RL_with_SRMs_bibliography}

@article{Corazza_2022,
   title={Reinforcement Learning with Stochastic Reward Machines},
   volume={36},
   ISSN={2159-5399},
   url={http://dx.doi.org/10.1609/aaai.v36i6.20594},
   DOI={10.1609/aaai.v36i6.20594},
   number={6},
   journal={Proceedings of the AAAI Conference on Artificial Intelligence},
   publisher={Association for the Advancement of Artificial Intelligence (AAAI)},
   author={Corazza, Jan and Gavran, Ivan and Neider, Daniel},
   year={2022},
   month=jun, pages={6429–6436} }

@inproceedings{KR2025-55,
    title     = {{Pushdown Reward Machines for Reinforcement Learning}},
    author    = {Varricchione, Giovanni and Klassen, Toryn Q. and Alechina, Natasha and Dastani, Mehdi and Logan, Brian and McIlraith, Sheila A.},
    booktitle = {{Proceedings of the 22nd International Conference on Principles of Knowledge Representation and Reasoning}},
    pages     = {566--576},
    year      = {2025},
    month     = {10},
    doi       = {10.24963/kr.2025/55},
    url       = {https://doi.org/10.24963/kr.2025/55},
  }

@inproceedings{
bester2023counting,
title={Counting Reward Automata: Sample Efficient Reinforcement Learning Through The Exploitation of Reward Function Structure},
author={Tristan Bester and Benjamin Rosman and Steven James and Geraud Nangue Tasse},
booktitle={Neuro-Symbolic Learning and Reasoning in the era of Large Language Models},
year={2023},
url={https://openreview.net/forum?id=XfXLM8uIxH}
}

@incollection{Puterman.1990,
 author = {Puterman, Martin L.},
 title = {Chapter 8 Markov decision processes},
 pages = {331--434},
 volume = {2},
 publisher = {North-Holland},
 isbn = {9780444874733},
 series = {Handbooks in Operations Research and Management Science},
 editor = {Heyman, Daniel P. and Sobel, Matthew J.},
 booktitle = {Stochastic models},
 year = {1990},
 address = {Amsterdam and New York and New York, NY, U.S.A},
 doi = {10.1016/S0927-0507(05)80172-0}
}

@book{Sutton.2020,
 author = {Sutton, Richard S. and Barto, Andrew},
 year = {2020},
 title = {Reinforcement learning: An introduction},
 price = {Hardcover : GBP 62.00, USD 80.00},
 address = {Cambridge, Massachusetts and London, England},
 edition = {Second edition},
 publisher = {{The MIT Press}},
 isbn = {9780262039246},
 series = {Adaptive computation and machine learning}
}

@inproceedings{
towers2025gymnasium,
title={Gymnasium: A Standard Interface for Reinforcement Learning Environments},
author={Mark Towers and Ariel Kwiatkowski and John U. Balis and Gianluca De Cola and Tristan Deleu and Manuel Goul{\~a}o and Kallinteris Andreas and Markus Krimmel and Arjun KG and Rodrigo De Lazcano Perez-Vicente and J K Terry and Andrea Pierr{\'e} and Sander V Schulhoff and Jun Jet Tai and Hannah Tan and Omar G. Younis},
booktitle={The Thirty-ninth Annual Conference on Neural Information Processing Systems Datasets and Benchmarks Track},
year={2025},
url={https://openreview.net/forum?id=qPMLvJxtPK}
}

@InProceedings{10.1007/978-3-642-39274-0_3,
author="Veanes, Margus",
editor="Konstantinidis, Stavros",
title="Applications of Symbolic Finite Automata",
booktitle="Implementation and Application of Automata",
year="2013",
publisher="Springer Berlin Heidelberg",
address="Berlin, Heidelberg",
pages="16--23",
isbn="978-3-642-39274-0"
}

@article{WatkinsChristopherJ.C.H..1992,
 author = {{Watkins, Christopher J. C. H.} and Dayan, Peter},
 year = {1992},
 title = {Q-learning},
 url = {https://link.springer.com/article/10.1007/BF00992698},
 pages = {279--292},
 volume = {8},
 number = {3-4},
 issn = {1573-0565},
 journal = {Machine Learning},
 doi = {10.1007/BF00992698}
}

@InProceedings{Moura.2008,
author="de Moura, Leonardo
and Bj{\o}rner, Nikolaj",
editor="Ramakrishnan, C. R.
and Rehof, Jakob",
title="Z3: An Efficient SMT Solver",
booktitle="Tools and Algorithms for the Construction and Analysis of Systems",
year="2008",
publisher="Springer Berlin Heidelberg",
address="Berlin, Heidelberg",
pages="337--340",
isbn="978-3-540-78800-3"
}

@misc{Moore.1990,
 author = {Moore, Andrew William},
 date = {1990},
 year = {1990},
 title = {Efficient memory-based learning for robot control},
 publisher = {{Computer Laboratory, University of Cambridge}},
 doi = {10.48456/TR-209}
}

@InProceedings{10.1007/978-3-319-45641-6_26,
author="Monniaux, David",
editor="Gerdt, Vladimir P.
and Koepf, Wolfram
and Seiler, Werner M.
and Vorozhtsov, Evgenii V.",
title="A Survey of Satisfiability Modulo Theory",
booktitle="Computer Algebra in Scientific Computing",
year="2016",
publisher="Springer International Publishing",
address="Cham",
pages="401--425",
isbn="978-3-319-45641-6"
}

@inproceedings{10.5555/3709347.3743526,
author = {Ardon, Leo and Furelos-Blanco, Daniel and Parac, Roko and Russo, Alessandra},
title = {FORM: Learning Expressive and Transferable First-Order Logic Reward Machines},
year = {2025},
isbn = {9798400714269},
publisher = {International Foundation for Autonomous Agents and Multiagent Systems},
address = {Richland, SC},
booktitle = {Proceedings of the 24th International Conference on Autonomous Agents and Multiagent Systems},
pages = {142–151},
numpages = {10},
location = {Detroit, MI, USA},
series = {AAMAS '25}
}

@incollection{DAntoni.2017,
 author = {D'Antoni, Loris and Veanes, Margus},
 title = {The Power of Symbolic Automata and Transducers},
 pages = {47--67},
 volume = {10426},
 publisher = {Springer},
 isbn = {978-3-319-63386-2},
 series = {Lecture Notes in Computer Science},
 editor = {Majumdar, Rupak and Kun{\v{c}}ak, Viktor},
 booktitle = {Computer aided verification},
 year = {2017},
 address = {Cham},
 doi = {10.1007/978-3-319-63387-9_3}
}

@article{Xu_Gavran_Ahmad_Majumdar_Neider_Topcu_Wu_2020, title={Joint Inference of Reward Machines and Policies for Reinforcement Learning}, volume={30}, url={https://ojs.aaai.org/index.php/ICAPS/article/view/6756}, DOI={10.1609/icaps.v30i1.6756}, abstractNote={&lt;p&gt;Incorporating &lt;em&gt;high-level knowledge&lt;/em&gt; is an effective way to expedite reinforcement learning (RL), especially for complex tasks with sparse rewards. We investigate an RL problem where the high-level knowledge is in the form of &lt;em&gt;reward machines&lt;/em&gt;, a type of Mealy machines that encode non-Markovian reward functions. We focus on a setting in which this knowledge is &lt;em&gt;a priori&lt;/em&gt; not available to the learning agent. We develop an iterative algorithm that performs joint inference of reward machines and policies for RL (more specifically, q-learning). In each iteration, the algorithm maintains a &lt;em&gt;hypothesis&lt;/em&gt; reward machine and a &lt;em&gt;sample&lt;/em&gt; of RL episodes. It uses a separate q-function defined for each state of the current hypothesis reward machine to determine the policy and performs RL to update the q-functions. While performing RL, the algorithm updates the sample by adding RL episodes along which the obtained rewards are inconsistent with the rewards based on the current hypothesis reward machine. In the next iteration, the algorithm infers a new hypothesis reward machine from the updated sample. Based on an &lt;em&gt;equivalence&lt;/em&gt; relation between states of reward machines, we transfer the q-functions between the hypothesis reward machines in consecutive iterations. We prove that the proposed algorithm converges almost surely to an optimal policy in the limit. The experiments show that learning high-level knowledge in the form of reward machines leads to fast convergence to optimal policies in RL, while the baseline RL methods fail to converge to optimal policies after a substantial number of training steps.&lt;/p&gt;}, number={1}, journal={Proceedings of the International Conference on Automated Planning and Scheduling}, author={Xu, Zhe and Gavran, Ivan and Ahmad, Yousef and Majumdar, Rupak and Neider, Daniel and Topcu, Ufuk and Wu, Bo}, year={2020}, month={Jun.}, pages={590-598} }

@article{Kaelbling.1996,
 author = {Kaelbling, L. P. and Littman, M. L. and Moore, A. W.},
 year = {1996},
 title = {Reinforcement Learning: A Survey},
 pages = {237--285},
 volume = {4},
 journal = {Journal of Artificial Intelligence Research},
 doi = {10.1613/jair.301}
}

@article{mnih2015humanlevel,
  author = {Mnih, Volodymyr and Kavukcuoglu, Koray and Silver, David and Rusu, Andrei A. and Veness, Joel and Bellemare, Marc G. and Graves, Alex and Riedmiller, Martin and Fidjeland, Andreas K. and Ostrovski, Georg and Petersen, Stig and Beattie, Charles and Sadik, Amir and Antonoglou, Ioannis and King, Helen and Kumaran, Dharshan and Wierstra, Daan and Legg, Shane and Hassabis, Demis},
  description = {Human-level control through deep reinforcement learning - nature14236.pdf},
  issn = {00280836},
  journal = {Nature},
  month = feb,
  number = 7540,
  pages = {529--533},
  publisher = {Nature Publishing Group, a division of Macmillan Publishers Limited. All Rights Reserved.},
  title = {Human-level control through deep reinforcement learning},
  url = {http://dx.doi.org/10.1038/nature14236},
  volume = 518,
  year = 2015
}

@inproceedings{Icarte.2018,
 author = {Icarte, Rodrigo Toro and Klassen, Toryn and Valenzano, Richard and McIlraith, Sheila},
 title = {Using Reward Machines for High-Level Task Specification and Decomposition in Reinforcement Learning},
 url = {https://proceedings.mlr.press/v80/icarte18a.html},
 pages = {2107--2116},
 volume = {80},
 publisher = {PMLR},
 series = {Proceedings of Machine Learning Research},
 editor = {Dy, Jennifer and Krause, Andreas},
 booktitle = {Proceedings of the 35th International Conference on Machine Learning},
 year = {2018}
}

@InProceedings{pmlr-v162-zhou22b,
  title = 	 {A Hierarchical {B}ayesian Approach to Inverse Reinforcement Learning with Symbolic Reward Machines},
  author =       {Zhou, Weichao and Li, Wenchao},
  booktitle = 	 {Proceedings of the 39th International Conference on Machine Learning},
  pages = 	 {27159--27178},
  year = 	 {2022},
  editor = 	 {Chaudhuri, Kamalika and Jegelka, Stefanie and Song, Le and Szepesvari, Csaba and Niu, Gang and Sabato, Sivan},
  volume = 	 {162},
  series = 	 {Proceedings of Machine Learning Research},
  month = 	 {17--23 Jul},
  publisher =    {PMLR},
  url = 	 {https://proceedings.mlr.press/v162/zhou22b.html}
}

@inproceedings{NEURIPS2019_bdbca288,
 author = {Paszke, Adam and Gross, Sam and Massa, Francisco and Lerer, Adam and Bradbury, James and Chanan, Gregory and Killeen, Trevor and Lin, Zeming and Gimelshein, Natalia and Antiga, Luca and Desmaison, Alban and Kopf, Andreas and Yang, Edward and DeVito, Zachary and Raison, Martin and Tejani, Alykhan and Chilamkurthy, Sasank and Steiner, Benoit and Fang, Lu and Bai, Junjie and Chintala, Soumith},
 booktitle = {Advances in Neural Information Processing Systems},
 editor = {H. Wallach and H. Larochelle and A. Beygelzimer and F. d\textquotesingle Alch\'{e}-Buc and E. Fox and R. Garnett},
 pages = {},
 publisher = {Curran Associates, Inc.},
 title = {PyTorch: An Imperative Style, High-Performance Deep Learning Library},
 url = {https://proceedings.neurips.cc/paper_files/paper/2019/file/bdbca288fee7f92f2bfa9f7012727740-Paper.pdf},
 volume = {32},
 year = {2019}
}

@article{10.1613/jair.1.12440,
author = {Toro Icarte, Rodrigo and Klassen, Toryn Q. and Valenzano, Richard and McIlraith, Sheila A.},
title = {Reward Machines: Exploiting Reward Function Structure in Reinforcement Learning},
year = {2022},
issue_date = {May 2022},
publisher = {AI Access Foundation},
address = {El Segundo, CA, USA},
volume = {73},
issn = {1076-9757},
url = {https://doi.org/10.1613/jair.1.12440},
doi = {10.1613/jair.1.12440},
journal = {J. Artif. Int. Res.},
month = may,
numpages = {36}
}

\end{document}